\documentclass{article}

\usepackage{microtype}
\usepackage{graphicx}
\usepackage{booktabs} 

\usepackage{amsmath, amssymb, amsthm}
\usepackage{dsfont}
\usepackage{mathtools}
\usepackage{subcaption}
\usepackage{tikz}
\usetikzlibrary{shapes.geometric}
\usepackage{standalone}

\usepackage{enumitem}

\usepackage{hyperref}

\newcommand{\rhomax}{\ensuremath{U_\rho}}
\newcommand{\wmax}{\ensuremath{U_s}}
\newcommand{\is}{\ensuremath{\hat{v}_\text{IS}}}
\newcommand{\pdis}{\ensuremath{\hat{v}_\text{PDIS}}}
\newcommand{\sis}{\ensuremath{\hat{v}_\text{SIS}}}
\newcommand{\var}{\text{Var}}
\newcommand{\cov}{\text{Cov}}
\newcommand{\dsda}{\text{d}s\text{d}a}
\newcommand{\TV}{d_\text{TV}}

\newtheorem{definition}{Definition}
\newtheorem{theorem}{Theorem}
\newtheorem{lemma}{Lemma}
\newtheorem{corollary}{Corollary}
\newtheorem{assumption}{Assumption}

\newcommand{\wt}{w_t}

\newcommand{\siswithwt}{\ensuremath{\hat{v}_\text{ASIS}}}

\newcommand{\mseofwt}{\epsilon_{w}}

\usepackage[accepted]{icml2020}

\icmltitlerunning{Understanding the Curse of Horizon in Off-Policy Evaluation via Conditional Importance Sampling}

\begin{document}

\twocolumn[
\icmltitle{Understanding the Curse of Horizon in Off-Policy Evaluation via Conditional Importance Sampling}



\icmlsetsymbol{equal}{*}

\begin{icmlauthorlist}
\icmlauthor{Yao Liu}{s}
\icmlauthor{Pierre-Luc Bacon}{m}
\icmlauthor{Emma Brunskill}{s}
\end{icmlauthorlist}

\icmlaffiliation{s}{Department of Computer Science, Stanford University}
\icmlaffiliation{m}{Mila - University
of Montreal. This work is finished when Pierre-Luc Bacon was a post-doc at Stanford.}

\icmlcorrespondingauthor{Yao Liu}{yaoliu@stanford.edu}

\icmlkeywords{Machine Learning, ICML}

\vskip 0.3in
]



\printAffiliationsAndNotice{} 

\begin{abstract}
Off-policy policy estimators that use importance sampling (IS) can suffer from high variance in long-horizon domains, and there has been particular excitement over new IS methods that leverage the structure of Markov decision processes. We analyze the variance of the most popular approaches through the viewpoint of conditional Monte Carlo. Surprisingly, we find that in finite horizon MDPs there is no strict variance reduction of per-decision importance sampling or stationary importance sampling, comparing with vanilla importance sampling. We then provide sufficient conditions under which the per-decision or stationary estimators will provably reduce the variance over importance sampling with finite horizons. For the asymptotic (in terms of horizon $T$) case, we develop upper and lower bounds on the variance of those estimators which yields sufficient conditions under which there exists an exponential v.s. polynomial gap between the variance of importance sampling and that of the per-decision or stationary estimators. These results help advance our understanding of if and when new types of IS estimators will improve the accuracy of off-policy estimation.
\end{abstract}

\section{Introduction}
Off-policy \citep{SuttonBarto2018} policy evaluation is the problem of estimating the expected return of a given \textit{target} policy from the distribution of samples induced by a different policy. Due in part to the growing sources of data about past sequences of decisions and their outcomes -- from marketing to energy management to healthcare -- there is increasing interest in developing accurate and efficient algorithms for off-policy policy evaluation. 

For Markov Decision Processes, this problem was addressed \citep{Precup2000,peshkin2002learning} early on by importance sampling (IS)\citep{Rubinstein1981}, a method prone to large variance due to rare events \citep{Glynn1994,LEcuyer2009}. The \textit{per-decision} importance sampling estimator of \cite{Precup2000} tries to mitigate this problem by leveraging the temporal structure -- earlier rewards cannot depend on later decisions -- of the domain. 

While neither importance sampling (IS) nor per-decision IS (PDIS) assumes the underlying domain is Markov, more recently, a new class of estimators \citep{Hallak2017,Liu2018,Gelada2019} has been proposed that leverages the Markovian structure. In particular, these approaches propose performing importance sampling over the stationary state-action distributions induced by the corresponding Markov chain for a particular policy. By avoiding the explicit accumulation of likelihood ratios along the trajectories, it is hypothesized that such ratios of stationary distributions could substantially reduce the variance of the resulting estimator, thereby overcoming the ``curse of horizon'' \citep{Liu2018} plaguing off-policy evaluation. The recent flurry of empirical results shows significant performance improvements over the alternative methods on a variety of simulation domains. Yet so far there has not been a formal analysis of the accuracy of IS, PDIS, and stationary state--action IS which will strengthen our understanding of their properties, benefits and limitations.

To formally understand the variance relationship between those unbiased estimators, we link this to a more general class of estimators: the \textit{extended} \citep{Bratley1987} form of the conditional Monte Carlo estimators \citep{Hammersley1956,Dubi1979,Granovsky1981}, and thus view those importance sampling estimators in a unified framework and referred to as conditional importance sampling, since they are computing weights as conditional expectation of likelihood ratio conditioning on different choice of statistics. Though the intuition from prior work suggests that stationary importance sampling should have the best accuracy, followed by per-decision IS and then (worst) the crude IS estimator. Surprisingly, we show that this is not always the case. In particular, we construct short-horizon MDP examples in Figure \ref{fig:counterexamples} that demonstrate that the crude IS can have a lower variance estimate than per-decision IS or stationary IS, and also show results for the other cases. 

We then describe how this observation is quite natural when we note that all three estimators are instances of conditional expectation. If $X$ and $Y$ are two well-defined random variables on the same probability space such that $\theta = \mathbb{E}\left[Y\right] = \mathbb{E}\left[\mathbb{E}[Y|X]\right]$, then the conditional Monte Carlo estimator for $\theta$ is
$\mathbb{E}\left[Y|x\right]$. By the law of total variance, the variance of the conditional Monte Carlo estimator cannot be larger than that of the crude Monte Carlo estimator $y$. However when $X$ and $Y$ are sequences of random variables, and we want to estimate $\mathbb{E}\left[\sum_{t=1}^T Y_t\right]$, the variance of the so-called \textit{extended} conditional Monte Carlo estimator $\sum_{t=1}^T \mathbb{E}\left[Y_t | x_t\right]$ is not guaranteed to reduce variance due to covariance between the summands. 

Building on these insights, we then provide a general variance analysis for conditional importance sampling estimators, as well as sufficient conditions for variance reduction in Section \ref{sect:increasing_condition}. In Section \ref{sect:asymp_infinite} we provide upper and lower bounds for the asymptotic variance of the crude, per-decision and stationary estimators. These bounds show, under certain conditions, that the per-decision and stationary importance sampling estimators can reduce the asymptotic variance to a polynomial function of the horizon compared to the exponential dependence of the per-decision estimator. Our proofs apply to general state spaces and use concentration inequalities for martingales. Importantly, these bounds characterize a set of common conditions under which the variance of stationary importance sampling can be smaller than that of per decision importance sampling, which in turn, can have a smaller variance than the crude importance sampling estimator. In doing so, our results provide concrete theoretical foundations supporting recent empirical successes in long-horizon domains.

\section{Notation and Problem Setting}
We consider Markov Decision Processes (MDPs) with discounted or undiscounted rewards and under a fixed horizon $T$. An MDP is defined by a tuple $(\mathcal{S},\mathcal{A},P,p_1,r,\gamma,T)$, where $\mathcal{S} \subset \mathbb{R}^d$ is the state space and $\mathcal{A}$ is the action space, which we assume are both bounded compact Hausdorf spaces. We use the notation $P(S|s,a)$ to denote the transition probability kernel where $S \subset \mathcal{S}, s \in \mathcal{S}, a\in \mathcal{A}$ and $r_t(s,a): \mathcal{S}\times\mathcal{A}\times[T] \mapsto [0,1]$ for the reward function. The symbol $\gamma \in [0,1]$
\footnote{Our analysis works for both discounted and undiscounted reward.} refers to the discount factor. For simplicity we write the probability density function associated with $P$ as $p(s'|s,a)$. Furthermore, our definition contains a probability density function of the initial state $s_1$ which we denote by $p_1$. 
We use $\pi(a_t|s_t)$ and $\mu(a_t|s_t)$ to denote the conditional probability density/mass functions associated with the policies $\pi$ and $\mu$. We call $\mu$ the \textit{behavior} policy and $\pi$ the \textit{target} policy. We assume $\frac{\pi(a|s)}{\mu(a|s)} < \infty$ throughout this paper which is the necessary condition for the effectiveness of all importance sampling bassed methods. We are interested in estimating the value of $\pi$, defined as:
\begin{align*}
    v^\pi = \mathbb{E}_{\pi} \left[ \sum_{t=1}^T \gamma^{t-1} r_t \right] \enspace .
\end{align*}
Furthermore, we use the notation $\tau_{1:T}$ to denote a $T$-step trajectory of the form: $\tau_{1:T} = \{(s_t,a_t,r_t)\}_{t=1}^T$. When appropriate, we use the subscript $\pi$ or $\mu$ to specify if $\tau_{1:T}$ comes from the induced distribution of $\pi$ or $\mu$. We use the convention that the lack of subscript for $\mathbb{E}$ is equivalent to writing $\mathbb{E}_\mu$, but otherwise write $\mathbb{E}_\pi$ explicitly. We denote the $1$-step likelihood ratio and the $T$-steps likelihood ratio respectively as: 
\begin{align*}
\rho_t = \frac{\pi(a_t|s_t)}{\mu(a_t|s_t)}, \quad \rho_{1:T} = \prod_{t=1}^T \rho_t \enspace.
\end{align*}
We define the $T$-step state distribution and stationary state distribution under the behavior policy as:
$$
    d^\mu_{t}(s,a) = \Pr(s_t = s, a_t = a |s_1 \sim p_1, a_i \sim \mu(a_i|s_i)) $$
$$
    d^\mu_{\gamma, 1:T}(s,a) = \frac{\sum_{t=1}^T \gamma^t d^\mu_{t}(s,a)}{\sum_{t=1}^T \gamma^t}, \,d^\mu_{\gamma} = \lim_{T \to \infty} d^\mu_{\gamma, 1:T}
$$
For simplicity of notation, we drop the $\gamma$ in $d^\mu_{\gamma, 1:T}$ and $d^\mu_\gamma$ when $\gamma=1$, and overload $d^\mu$ to denote the marginal state distribution as well: ie. $d^\mu(s)=\int_a d^\mu(s,a) \text{d}a$ (and similarly for $d^\pi$). We use $c$ to denote the KL divergence of $\mu$ and $\pi$ and where the expectation is taken under $d^{\mu}$ over the states: $\mathbb{E}_{d^{\mu}}\left[ D_{\text{KL}}(\mu||\pi) \right]$. We assume $c>0$, otherwise $\pi$ and $\mu$ are identical and our problem reduces to on-policy policy evaluation. 

In this paper, we define the estimator and discuss the variance over a single trajectory but of all our results carry to $N$ trajectories by multiplying by a factor $1/N$. 
We define the \textit{crude} importance sampling (IS) estimator and the per-decision importance sampling (PDIS) from \cite{Precup2000} as:
\begin{align*}
    \is = \rho_{1:t}\sum_{t=1}^T \gamma^{t-1} r_t, \enspace \pdis = \sum_{t=1}^T \gamma^{t-1} r_t \rho_{1:t} \enspace. 
\end{align*}
The stationary importance sampling (SIS) estimator is defined as:
\begin{align*}
    \sis = \sum_{t=1}^T \gamma^{t-1} r_t \frac{d_{t}^\pi(s_t,a_t)}{d_{t}^\mu(s_t,a_t)} \enspace .
\end{align*}
All three estimators are unbiased. Our definition of SIS is based on the importance ratio of the time-dependent state distributions (provided by an oracle) rather than the stationary distributions as in the prior work by \cite{Liu2018,Hallak2017}, but similar with \cite{Xie2019}. This choice allows us to tackle both the finite horizon and infinite horizon more easily under the same general framework by taking $T \to \infty$ when necessary. This is possible because the ratio $\frac{d^\pi(s_t,a_t)}{d^\mu(s_t,a_t)}$ has the same asymptotic behavior as that of the stationary distribution ratio. Surprisingly, we show that even under perfect knowledge of the stationary ratio, it is generally non-trivial to guarantee a variance reduction for $\sis$.

The next standard assumptions helps us analyze our estimators in the asymptotic regime by relying on a central limit property for general Markov chains.
\begin{assumption}[Harris ergodic] \label{assum:harris} 
The Markov chain of $\{s_t,a_t\}$ under $\mu$ is Harris ergodic. That is: the chain is aperiodic, $\psi$-irreducible, and positive Harris recurrent. See \cite{meyn2012markov} for more
\end{assumption}
\begin{assumption}[Drift property]\label{assum:drift}
There exist an everywhere-finite function $B: \mathcal{S}\times\mathcal{A} \mapsto [1,\infty)$, a constant $\lambda \in (0,1)$, $b < \infty$ and a petite $K \subset \mathcal{S}\times\mathcal{A}$ such that:
\begin{align*}
    \mathbb{E}_{s',a'|s,a}B(s',a') \le \lambda B(s,a) + b\mathds{1}((s,a) \in K) \enspace .
\end{align*}
\end{assumption}
These are standard assumptions to describe the ergodic and recurrent properties of general Markov chains. Assumption \ref{assum:harris} is typically used to obtain the existence of a unique stationary distribution \citep{meyn2012markov} and assumption \ref{assum:drift} is used to measure the concentration property \citep{meyn2012markov, jones2004markov}.

\begin{figure*}[tb!]
    \begin{minipage}{0.31\textwidth}
    \includegraphics[width=\columnwidth]{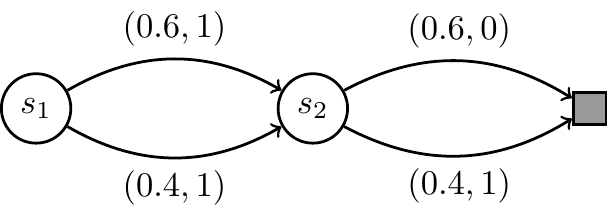}
    \subcaption{{\scriptsize $\var(\is) < \var(\sis) < \var(\pdis)$}}
    \label{fig:example1}
    \end{minipage}
    \hspace{\fill}
    \begin{minipage}{0.31\textwidth}
    \includegraphics[width=\columnwidth]{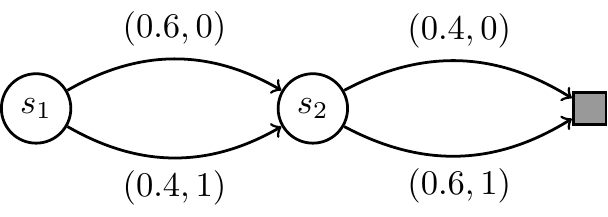}
    \subcaption{{\scriptsize$\var(\pdis) < \var(\sis) < \var(\is)$}}
    \label{fig:example2}
    \end{minipage}
    \hspace{\fill}
    \begin{minipage}{0.31\textwidth}
    \includegraphics[width=\columnwidth]{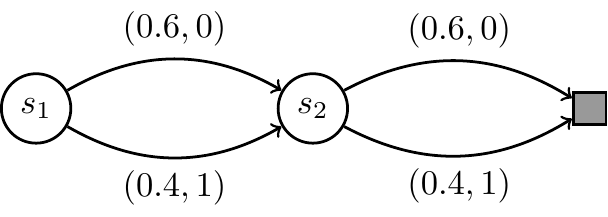}
    \subcaption{{\scriptsize$\var(\is) < \var(\pdis) < \var(\sis)$}}
    \label{fig:example3}
    \end{minipage}
    \caption{Counterexamples. The labels for each edge are of the form \texttt{(target policy probability, reward)} where the first component is the transition probability induced by the given target policy, and the second component is the reward function for this transition. All examples assume deterministic transitions and the same initial state $s_1$. The square symbol represents a terminating state.}
    \label{fig:counterexamples}
\end{figure*}

\section{Counterexamples}
\label{sect:counterexamples}

It is tempting to presume that the root cause of the variance issues in importance sampling pertains entirely to the explicit \textit{multiplicative} \citep{Liu2018} accumulation of importance weights over long trajectories. The reasoning ensuing from this intuition is that the more terms one can drop off this product, the better the resulting estimator would be in terms of variance. In this section, we show that this intuition is misleading as we can construct small MDPs in which per-decision or stationary importance sampling does not necessarily reduce the variance of the crude importance sampling. We then explain this phenomenon in section \ref{sect:cis} using the extended conditional Monte Carlo method and point out that the lack of variance reduction is attributable to the interaction of some covariance terms across time steps. However, all is not lost and section \ref{sect:asymp_infinite} shows that asymptotically ($T \to \infty$) stationary importance sampling achieves much lower variance than crude importance sampling or the per-decision variant. 

In all examples, we use a two-steps MDP with deterministic transitions, undiscounted reward, and in which a uniform behavior policy is always initialized from the state $s_1$ (see figure \ref{fig:counterexamples}). We then show that the ordering of the estimators based on their variance can vary by manipulating the target policy and the reward function so as to induce a different covariance structure between the reward and the likelihood ratio. We can then compute the exact variance (table \ref{table:variances}) of each estimator manually (see appendix \ref{appd:counterexamples_details}). Example \ref{fig:example1} shows that the per-decision estimator can have a larger variance than the crude estimator when stationary estimator improves on per-decision estimator. Example \ref{fig:example2} shows an instance where the stationary estimator does not improve on the per-decision importance sampling, but per-decision importance sampling has a smaller variance than crude importance sampling. Finally, example \ref{fig:example3} provides a negative example where the ordering goes against our intuition and shows that the stationary estimator is worse than the per-decision estimator, which in turn has a larger variance than the crude estimator. Note that the lack of variance reduction for stationary IS occurs even with perfect knowledge of the stationary ratio. We show in section \ref{sect:cis} that the problem comes from the covariance terms across time steps.

\begin{table}[h]
\begin{center}
\begin{tabular}{llclcl}
\toprule
         &IS & & PDIS  & & SIS     \\
\midrule
Example \ref{fig:example1} & 0.12 & $<$ & 0.2448 & $>$ &0.2  \\
Example \ref{fig:example2} & 0.5424 & $>$ & 0.4528 & $<$ & 0.52  \\
Example \ref{fig:example3} & 0.2304 & $<$ & 0.2688 & $<$ & 0.32 \\
\bottomrule
\end{tabular}
\end{center}
\caption{Analytical variance of different estimators. See figure \ref{fig:counterexamples} for the problem structure.}
\label{table:variances}
\end{table}

\section{Conditional Importance Sampling}
\label{sect:cis}
The unbiasedness of crude importance sampling (IS) estimator follows from the fact that:
\begin{align*}
     \mathbb{E}\left[\rho_{1:t}\sum_{t=1}^T \gamma^{t-1} r_t \right] = \mathbb{E}_\pi\left[\sum_{t=1}^T \gamma^{t-1} r_t \right] = v^\pi\enspace,
\end{align*}
Let $G_T$ be the total (discounted) return $\sum_{t=1}^T \gamma^{t-1} r_t $ and if $\phi_T$ is some statistics such that $\rho_{1:T}$ is conditionally independent with $G_T$ given $\phi_T$, then by the law of total expectation:
\begin{align*}
\mathbb{E}\left[ \rho_{1:T}G_T \right] &= \mathbb{E}\left[ \mathbb{E}\left[\rho_{1:T} G_T |\phi_T, G_T\right]\right] \\
&= \mathbb{E}\left[ G_T \mathbb{E}\left[ \rho_{1:T} | \phi_T, G_T\right]\right] \\
&= \mathbb{E}\left[  G_T \mathbb{E}\left[ \rho_{1:T} | \phi_T \right]\right] .
\end{align*}
Furthermore, by the law of total variance we have:
\begin{align*}
&\var \left(G_T \mathbb{E}\left[ \rho_{1:T} | \phi_T\right]\right)\\
&= \var\left(G_T\rho_{1:T}\right) - \mathbb{E}\left[\var(G_T\phi_{1:T}|\phi_T,G_T)\right]\\
&=\var\left(G_T\rho_{1:T}\right) - \mathbb{E}\left[G_T^2 \var(\rho_{1:T}|\phi_T)\right] \enspace .
\end{align*}
Because the second term is always non-negative, it follows that $\var\left(G_T \mathbb{E}\left[ \rho_{1:T} | \phi_T\right]\right) \leq \var\left(G_T\rho_{1:T}\right)$. This conditioning idea is the basis for the conditional Monte Carlo (CMC) as a variance reduction method. 

If we now allow ourselves to conditions in a stage-dependent manner rather than with a fixed statistics $\phi_T$, we obtain estimators belonging to the so-called \textit{extended} conditional Monte Carlo methods \citep{Bratley1987}. Assuming that $r_t$ is conditionally independent with $\rho_{1:T}$ given $\phi_t$, then by the law of total expectation:
\begin{align*}
    v^\pi = \mathbb{E}\left[ G_T \rho_{1:T}\right]
    &=\sum_{t=1}^T \gamma^{t-1}\mathbb{E}\left[\mathbb{E}\left[r_t \rho_{1:t}|\phi_t, r_t\right] \right] \\
    &= \mathbb{E}\left[\sum_{t=1}^T \gamma^{t-1} r_t \mathbb{E}\left[ \rho_{1:t}|\phi_t\right] \right] \enspace.
\end{align*}
We refer to estimators in this form as ``\textit{extended} conditional importance sampling estimators'': a family of estimators encompassing both the per-decision importance sampling (PDIS) estimator of \cite{Precup2000} as well as the \textit{stationary} variants \citep{Hallak2017,Liu2018,Gelada2019}. In this paper, we use ``conditional importance sampling''\footnote{\cite{Bucklew2004,bucklew2005conditional} also uses this expression to describe the ``g-method'' of \cite{Srinivasan1998}, which is CMC applied to IS. Our work considers the extended form of the CMC method for Markov chains: a more general setting with very different variance properties.} to refer to all variants of importance sampling based on the conditional Monte Carlo method, in its ``extended'' form or not. 
To obtain the per-decision estimator in our framework, it suffices to define the stage-dependent statistics $\phi_t$ to be the history $\tau_{1:t}$ up to time $t$:
\begin{align*}
 v^\pi &= \mathbb{E}\left[\sum_{t=1}^T \gamma^{t-1} r_t \mathbb{E}\left[ \rho_{1:t}|\tau_{1:t}\right] \right] = \mathbb{E}\left[ \sum_{t=1}^T \gamma^{t-1} r_t \rho_{1:t} \right] \enspace .
\end{align*}
In this last expression, $\mathbb{E}[\rho_{1:T} | \tau_{1:t}] = \rho_{1:t}$ follows from the fact that the likelihood ratio is a martingale \citep{LEcuyer2008}. Similarly, the stationary importance sampling (SIS) estimator can be derived by conditioning on the state and action at time $t$:
\begin{align*}
    v^\pi = \mathbb{E}\left[\sum_{t=1}^T \gamma^{t-1} r_t \mathbb{E}\left[ \rho_{1:t}|s_t,a_t\right] \right] \\ = \mathbb{E}\left[\sum_{t=1}^T \gamma^{t-1} r_t \frac{d_{t}^\pi(s_t,a_t)}{d_{t}^\mu(s_t,a_t)} \right].
\end{align*}

In this case, the connection between the expected importance sampling weights conditioned on $(s_t, a_t)$ and the ratio of stationary distributions warrants a lengthier justification which we formalize in the following lemma (proved in appendix).
\begin{lemma}
\label{lem:conditional_expectation_given_sa}
$\mathbb{E} (\rho_{1:t}|s_t,a_t) = \frac{d^\pi(s_t,a_t)}{d^\mu(s_t,a_t)}$ \enspace .
\end{lemma}

Assuming that an unbiased estimator of the conditional weights $\mathbb{E}\left[ \rho_{1:t}|\phi_t\right]$ is available, the conditional importance sampling estimators are also unbiased. However, the law of total variance no longer implies a variance reduction because the variance is now over a sum of random variables of the form:
\begin{align*}
\var\left(\sum_{t=1}^T r_t w_t\right) = \sum_{t=1}^T \var(r_t w_t) + \sum_{k \not = t} \cov(r_k w_k,r_t w_t)\enspace.
\end{align*}
where $w_t = \mathbb{E}\left[\gamma^{t-1} \rho_{1:t}|\phi_t\right]$. In general, there is no reason to believe that the sum of covariance terms interact in such a way as to provide a variance reduction. If stage-dependent conditioning of the importance weights need not reduce the variance in general, all we are left with is to ``optimistically'' \citep{Bratley1987} suppose that the covariance structure plays in our favor. Over the next sections, we develop sufficient conditions for variance reduction in both the finite and infinite horizon setting. More specifically, theorem \ref{thm:episodic_pdis_variance} provides sufficient conditions for a variance reduction with the per-decision estimator while theorem \ref{thm:episodic_sis_variance} applies to the stationary importance sampling estimator. In section \ref{sect:asymp_infinite}, we develop an asymptotic analysis of the variance when $T \to \infty$. We show that under some mild assumptions, the variance of the crude importance sampling estimator is always exponentially large in the horizon $T$. Nevertheless, we show that there are cases where the per-decision or stationary estimators can help reduce the variance to $O(T^2)$. 

\section{Finite-Horizon Analysis}
\label{sect:increasing_condition}
While the counterexamples of section \ref{sect:counterexamples} show that there is no consistent order in general between the different IS estimators and their variance, we are still interested in characterizing when a variance reduction can occur. In this section, we provide theorems to answer when $\var(\pdis)$ is guaranteed to be smaller than $\var(\is)$ and when $\var(\sis)$ is guaranteed to be smaller than $\var(\pdis)$. We start by introducing a useful lemma to analyze the variance of the sum of conditional expectations.
\begin{lemma}
\label{lem:variance_of_sum}
Let $X_t$ and $Y_t$ be two sequences of random variables. Then
\begin{align*}
    &\var\left(\sum_t Y_t\right) - \var \left( \sum_t \mathbb{E}[Y_t|X_t] \right) \\
    &\ge 2  \sum_{t<k} \mathbb{E}[Y_t Y_k] - 2 \sum_{t<k}\mathbb{E}[ \mathbb{E}[Y_t|X_t] \mathbb{E}[Y_k|X_k]] \enspace .
\end{align*}
\end{lemma}
This lemma states that the variance reduction of the stage-dependent conditional expectation depends on the difference between the covariance of the random variables and that of their conditional expectations. The variance reduction analysis of PDIS and SIS in theorems \ref{thm:episodic_pdis_variance} and \ref{thm:episodic_sis_variance} can be viewed as a consequence of this result. We develop in those cases some sufficient conditions to guarantee that the difference between the covariance terms is positive.

\begin{theorem}[Variance reduction of PDIS] 
\label{thm:episodic_pdis_variance}
If for any $1 \le t \le k \le T$ and initial state $s$, $\rho_{0:k}(\tau)$ and $r_t(\tau)\rho_{0:k}(\tau)$ are positively correlated, $\var(\pdis) \le \var(\is)$.
\end{theorem}

This theorem guarantees the variance reduction of PDIS given a positive correlation between the likelihood ratio and the importance-weighted reward. The random variables $\rho_{0:k}(\tau)$ and $r_t(\tau)\rho_{0:k}(\tau)$ are positively correlated when for a trajectory with large likelihood ratio, the importance-weighted reward (which is an unbiased estimator of reward under the target policy $\pi$) is also large. Intuitively, a positive correlation is to be expected if the target policy $\pi$ is more likely to take a trajectory with a higher reward. We expect that this property may hold in applications where the target policy is near the optimal value for example.

\begin{theorem}[Variance reduction of SIS] 
\label{thm:episodic_sis_variance}
If for any fixed $0 \le t \le k < T$, 
$$\cov \left( \rho_{1:t} r_t,  \rho_{0:k}r_k \right) \ge \cov \left( \frac{d^\pi_t(s,a)}{d^\mu_t(s,a)} r_t, \frac{d^\pi_k(s,a)}{d^\mu_k(s,a)} r_k \right) $$
then $\var(\sis) \le \var(\pdis)$
\end{theorem}
This theorem implies that the relative order of variance between SIS and PDIS depends on the ordering of the covariance terms between time-steps. In the case when $T$ is very large, the covariance on the right is very close to zero, and if the covariance on the left is positive (which is true for many MDPs) the variance of SIS can be smaller than PDIS.

\section{Asymptotic Analysis}
\label{sect:asymp_infinite}
We have seen in section \ref{sect:counterexamples} that a variance reduction cannot be guaranteed in the general case and we then proceeded to derive sufficient conditions. However, this section shows that the intuition behind per-decision and stationary importance sampling does hold under some conditions and in the limit of the horizon $T \to \infty$. Under the light of these new results, we expect those estimators to compare favorably to crude importance sampling for very long horizons: an observation also implied by the sufficient conditions derived in the last section. 

In the following discussion, we consider the asymptotic rate of the variance as a function when $T \to \infty$. We show that under some mild assumptions, the variance of crude importance sampling is exponential with respect to $T$ and bounded from two sides. For the per-decision estimator, we provide conditions when the variance is at least exponential or at most polynomial with respect to $T$. Under some standard assumptions, we also show that the variance of stationary importance sampling can be polynomial with respect to $T$, indicating an exponential variance reduction. As a starting point, we prove a result characterizing the asymptotic distribution of the importance-weighted return. 
\begin{theorem} 
\label{thm:likelihood_ratio_distr}
Under Assumption \ref{assum:harris}, if $\log(\frac{\pi(a|s)}{\mu(a|s)})$ is a continuous function of $(s,a)$ in the support of $\mu$ then for $\pi \neq \mu$,
$\lim_{T} \left(\rho_{1:T}\right)^{1/T} = e^{-c}$,
$\overline{\lim}_{T} | \is |^{1/T} < e^{-c}$ a.s.
\end{theorem}

\begin{corollary}
\label{cor:likelihood_ratio_distr}
Under the same condition as theorem \ref{thm:likelihood_ratio_distr}, $\rho_{1:T} \rightarrow_{a.s.} 0$, $\rho_{1:T} \sum_{t=1}^T \gamma^{t-1} r_t  \rightarrow_{a.s.} 0$
\end{corollary}
Although crude importance sampling is unbiased, this result shows that it also converges to zero almost surely. Theorem \ref{thm:likelihood_ratio_distr} further proves that it converges to an exponentially small term $\exp(-cT)$. This indicates that in most cases the return is almost zero, leading to poor estimates of $v^\pi$, and under some rare events the return can be very large and the expectation is $v^\pi > 0$.

Equipped with these results, we can now show that the variance of the crude importance sampling estimator is exponential with respect to $T$. To quantitatively describe the variance, we need the following assumptions so that $\log \rho_t$ is bounded:
\begin{assumption}
\label{assum:bounded_ratio}
$|\log \rho_t| < \infty$
\end{assumption}
This assumption entails that $\rho_t$ is both upper-bounded (a common assumption) and lower-bounded. We only need the assumption on the lower bound of $\rho_t$ in the proof of a lower bound part in theorem \ref{thm:is_variance} and \ref{thm:pdis_variance}. For the lower bound part, it essentially amounts to the event where all likelihood ratio terms on a trajectory are greater than zero. Then by the law of total variance, the original variance can only be larger than the variance of all returns conditioned on this event. Before we characterize the variance of the IS estimator, we first prove that the log-likelihood ratio is a martingale with bounded differences.

\begin{lemma}
\label{lem:likelihood_ratio_martingale}
Under Assumption \ref{assum:harris}, \ref{assum:drift} and \ref{assum:bounded_ratio}, there exists a function $\hat{f}: \mathcal{S}\times\mathcal{A} \mapsto \mathbb{R}$ such that:
\begin{enumerate}
    \item $\forall (s,a), \, |\hat{f}(s,a)| < c_1\sqrt{B(s,a)}$ for constant $c_1$.
    \item For any $T>0$, $\log \rho_{1:T} + Tc - \hat{f}(s_1,a_1) + \hat{f}(s_{T+1},a_{T+1})$ is a mean-zero martingale with respect to the sequence $\{s_t,a_t\}_{t=1}^T$ with martingale differences bounded by $2c_1\sqrt{||B||_\infty}$.
\end{enumerate}
\end{lemma}

We are now ready to give both upper and lower bounds on the variance of the crude importance sampling estimator using an exponential function of $T$ from both sides.
\begin{theorem}[Variance of IS estimator]
\label{thm:is_variance}
Under Assumption \ref{assum:harris}, \ref{assum:drift} and \ref{assum:bounded_ratio}, there exist $T_0 > 0$ such that for all $T > T_0$,
$$\var(\is) \ge \frac{(v^\pi)^2}{4} \exp \left( \frac{Tc^2}{8c_1^2 \| B \|_\infty} \right) - (v^\pi)^2 $$
where $B$ is defined in Assumption \ref{assum:drift}, $c_1$ is some constant defined in lemma \ref{lem:likelihood_ratio_martingale}, $c = \mathbb{E}_{d^\mu} [D_{\text{KL}}(\mu||\pi)]$. If $\mathbb{E}_{a \sim \mu} \left[ \frac{\pi(a|s)^2}{\mu(a|s)^2} \right] \le M_{\rho}^2 $ for any s, then $\var\left(\is\right) \le T^2 M^{2T} - (v^\pi)^2 $.
\end{theorem}
The lower bound part shows that the variance is at least an exponential function of the horizon $T$, and the rate depends on the distance between the behavior and target policies, as well as the recurrent property of the Markov chain associated with the behavior policy. This result differs from that of \cite{Xie2019}, which is based on the CLT for i.i.d sequences, since our analysis considers more broadly a distribution of samples from a Markov chain.

\textit{Proof Sketch.} Let $Y$ be the IS estimator and $Z$ be indicator function $\mathds{1}(Y > v^{\pi}/2)$. By the law of total variance, $\var(Y) \ge \var(\mathbb{E}(Y|Z))$. Since the expectation of $\mathbb{E}(Y|Z)$ is a constant, we only need to show that the second moment of $\mathbb{E}(Y|Z)$ is asymptotically exponential. To achieve this, we observe that $\mathbb{E}[(\mathbb{E}[Y|Z])^2] \ge \Pr(Y > v^{\pi}/2)(\mathbb{E}[Y|Y>v^{\pi}/2])^2 $. We can then establish that $\mathbb{E}[Y|Y>v^{\pi}/2]$ is $\Omega(1/\Pr(Y > v^{\pi}/2))$ using the fact that the expectation of $Y$ is a constant. It follows that we can upper bound $\Pr(Y > v^{\pi}/2)$ by an exponentially small term. This can be done by a concentration inequality for martingales. The upper bound part is proved by bounding the absolute range of each variable. \hfill\ensuremath{\square}

Now we prove upper and lower bounds for the variance of the per-decision estimator as a function of $\gamma$, the expected reward at time $t$, $\mathbb{E}_\pi[r_t]$, and other properties of MDP. We then give a sufficient condition for the variance of PDIS to have an exponential lower bound, and when it is at most polynomial. 
\begin{theorem}[Variance of the PDIS estimator]
\label{thm:pdis_variance}
Under Assumption \ref{assum:harris}, \ref{assum:drift} and \ref{assum:bounded_ratio}, $\exists T_0 > 0$ s.t. $\forall T > T_0$,
\begin{align*}
     \var(\pdis) \ge \sum_{t=T_0}^T  \frac{\gamma^{2t-2}(\mathbb{E}_{\pi}(r_t))^2}{4} \exp \left( \frac{tc^2}{8c_1^2 \| B \|_\infty} \right)\\
     - (v^\pi)^2 
\end{align*}
where $B$ $c_1$ and $c$ are same constants in theorem \ref{thm:is_variance}, and $C$ is some constant. For the upper bound:
\begin{enumerate}
    \item If $\mathbb{E}_{a \sim \mu} \left[ \frac{\pi(a|s)^2}{\mu(a|s)^2} \right] \le M_{\rho}^2 $ for any s, $\var(\pdis) \le T\sum_{t=1}^T M_{\rho}^{2t} \gamma^{2t-2} - (v^\pi)^2$.
    \item Let $\rhomax = \sup_{s,a} \frac{\pi(a|s)}{\mu(a|s)} < \infty$,
$\var(\pdis) \le T\sum_{t=1}^T \rhomax^{2t} \gamma^{2t-2}\mathbb{E}_{\mu}[r_t^2] - (v^\pi)^2$.
\end{enumerate}
\end{theorem}
\textit{Proof Sketch.} The proof of the lower bound part is similar to the proof of the last theorem where we first lower bound the square of the sum by a sum of squares. We then apply the proof techniques of theorem \ref{thm:is_variance} for the time-dependent terms. The proof for the upper bound relies the Cauchy-Schwartz inequality on the square of sum and then upper bound each term directly. \hfill\ensuremath{\square}

Using theorem \ref{thm:pdis_variance}, we can now give sufficient conditions for the variance of the PDIS estimator to be at least exponential or at most polynomial.
\begin{corollary}
\label{cor:pdis_variance_lowerbound}
With theorem \ref{thm:pdis_variance} holds, $\var(\pdis ) = \Omega(\exp(\epsilon T))$ if the following conditions hold:
1) $\gamma \ge \exp \left( \frac{-c^2}{16c_1^2\|B\|_\infty}\right) $; 2) There exist a $\epsilon > 0$ such that 
$$\mathbb{E}_\pi(r_t) = \Omega\left(\exp \left( -t \left(\frac{c^2}{16c_1^2 \| B \|_\infty}  + \log\gamma - \epsilon/2 \right) \right)\right) $$
\end{corollary}
This corollary says that if $\gamma$ is close enough to $1$ and the expected reward under the target policy is larger than an exponentially decaying function, then the variance of \pdis is still at least exponentially large. We note that the second condition is satisfied if $r_t(s,a)$ is a function that does not depends on $t$ and  $\mathbb{E}_{d^\pi}(r(s,a)) > 0$. This is due to the fact that $\mathbb{E}_{\pi}(r_t) \to \mathbb{E}_{d^\pi}(r(s,a))$ as $t \to \infty$ and we obtain a constant which is larger than any exponentially decaying function.

\begin{corollary}
\label{cor:pdis_variance_upperbound}
Let $\rhomax = \sup_{s,a} \frac{\pi(a|s)}{\mu(a|s)}$. If $\rhomax \gamma \le 1$ or $\rhomax \gamma \lim_{T} \left(\mathbb{E}_{\pi}[r_T]\right)^{1/T} < 1$, $\var(\pdis) = O(T^2) $.
\end{corollary}
This corollary says that when $\gamma$ and the reward $\mathbb{E}_{\pi}(r_t)$ decreases fast enough, the variance of PDIS is polynomial in $T$, indicating an exponential improvement over crude importance sampling for long horizons. We can now prove an upper bound on the variance of stationary importance sampling. 
\begin{theorem}[Variance of the SIS estimator]
\label{thm:sis_variance}
$$\var \left(\sis \right) \le T \sum_{t=1}^T \gamma^{t-1} \left( \mathbb{E}\left[\left( \frac{d^{\pi}_t(s_t,a_t)}{d^{\mu}_t(s_t,a_t)} \right)^2\right] - 1 \right)$$
\end{theorem}
The proof uses Cauchy-Schwartz to bound each covariance term. In this theorem, the left hand side, is very close to $O(T^2)$ but $\mathbb{E}\left[\left( \frac{d^{\pi}_t(s_t,a_t)}{d^{\mu}_t(s_t,a_t)} \right)^2\right]$ still depends on $t$ and is not constant. Intuitively, the assumption that the ratio of stationary distributions is bounded is enough for this to hold since $d^{\mu}_t$ and $d^{\pi}_t$ is close to $d^\mu$ and $d^\pi$ for large $t$. We formally show this idea in the next corollary. However, we first need to introduce a continuity definition for function sequences.
\begin{definition}[asymptotically equi-continuous]
\label{def:aec}
A function sequence $f_t: \mathbb{R}^d \mapsto \mathbb{R}$ is asymptotically equi-continuous if for ant $\epsilon > 0$ there exist $n, \delta > 0$ such that for all $t>n$ and $\text{dist}(x_1, x_2) \le \delta$, $|f_t(x_1) - f_t(x_2)| \le \epsilon$
\end{definition}

\begin{corollary}
\label{cor:sis}
If $d^\mu_t(s_t)$ and $d^\pi_t(s_t)$ are asymptotically equi-continuous, $\frac{d^\pi(s)}{d^\mu(s)} \le \wmax$, and $\frac{\pi(a|s)}{\mu(a|s)} \le \rhomax$, then $\var \left(\sis \right) = O(T^2)$
\end{corollary}
This corollary implies that as long as the stationary ratio and one step ratios are bounded, the variance of stationary IS is $O(T^2)$. This result is predicated on having access to an oracle of $d^\pi_t/d^\mu_t$ because our results characterizes the variance reduction due to \textit{conditioning} irrespective of the choice of estimators for $d^\pi_t/d^\mu_t$. For general case of approximating choice of the ratio $d^\pi_t/d^\mu_t$ by a function $\wt (s_t,a_t)$, we could show an plug-in type estimator and a variance upper bound based on the Theorem above.

Now we consider approximate SIS estimators, which approximate density ratio $\frac{d^{\pi}_t(s_t,a_t)}{d^{\mu}_t(s_t,a_t)}$ by $\wt(s_t,a_t)$ and plug it into the SIS estimator. More specifically, 
\begin{align}
    \siswithwt = \sum_{t=1}^T \gamma^{t-1}\wt(s,a) r_t
\end{align}
This approximate SIS estimator is often biased based on the choice of $\wt(s,a)$, so we consider the upper bound of their mean square error with respect to $T$ and the error of the ratio estimator. We can show the following bound under the same condition as the oracle ratio case:
\begin{corollary}
\label{cor:asis}
Under the same condition of Corollary \ref{cor:sis}, $\siswithwt$ with $\wt$ such that where $ \mathbb{E}_{\mu} \left(\wt(s_t,a_t) - \frac{d^{\pi}_t(s_t,a_t)}{d^{\mu}_t(s_t,a_t)} \right)^2 \le \mseofwt $ has a MSE of $O(T^2(1+\mseofwt))$
\end{corollary}
Different accuracy bound of $\wt$ result in estimators with different variance. Previous work~\cite{Xie2019} shows the existence of $\wt$ estimator with polynomial MSE. This bound match the dependency on the horizon $O(T^3)$ from~\cite{Xie2019} for an $O(T)$ accurate $\wt$, with our proof considers general spaces with samples coming from a Markov chain, and potentially works for more general choice of $\wt$. This result, along with the lower bound for variance of PDIS and IS, suggests that for long-horizon problems SIS reduces the variance significantly, from $\exp(T)$ to $O(T^2)$. Only when corollary \ref{cor:pdis_variance_upperbound} holds, which requires a much stronger assumption than this, PDIS yields $O(T^2)$ variance.

There might also be question on if SIS estimators can potentially achive a better error bound than $O(T^2)$. We demonstrate an example to show that for any $T$ even with an oracle of the ratio $\frac{d^{\pi}_t(s_t,a_t)}{d^{\mu}_t(s_t,a_t)}$ or $\frac{d^\pi(s)}{d^\mu(s)}$, there exist an MDP such that $\var(\sis)$ is at least $\Theta(T^2)$.

\begin{definition}
\label{def:two_lane_mdp}
Given any $T>3$, define an MDP and off-policy evaluation problem in the following way:
\begin{enumerate}[nolistsep]
    \item There are two actions $a_1$ and $a_2$ in initial state $s_0$, leading to $s_1$ and $s_2$ separately. After $s_1$, the agent will go through $s_3, s_5, \dots s_{2n-1}$ sequentially, no matter which action was taken in $s_1, s_3, s_5, \dots s_{2T-1}$. Similarly, $s_2$ leads to a chain of $s_4, s_6, \dots s_{2T}$.
    \item The reward for $s_0,a_1$ and any action on $s_1, s_3, s_5, \dots s_{2n-1}$ is one, and zero otherwise.
    \item Behavior policy is an uniformly random policy, gives $0.5$ probability to go through each of the chains. Evaluation policy will always choose $a_1$ which leads to the chain of $s_1$.
\end{enumerate}
\end{definition}
In this example, it is easy to verify that the distribution of SIS given oracle $\frac{d^{\pi}_t(s_t,a_t)}{d^{\mu}_t(s_t,a_t)}$ is a uniform distribution over $\{0, 2T\}$, and the variance is $T^2$.


\section{Related Work}

This idea of substituting the importance ratios for their conditional expectations can be found in the thesis of \cite{Hesterberg1988} under the name \textit{conditional weights} and is presented as an instance of the conditional Monte Carlo method. Here instead, we consider the class of importance sampling estimators arising from the extended conditional Monte Carlo method and under a more general conditional independence assumption than that of \citep[p.48]{Hesterberg1988}. The ``conditional'' form of the per-decision and stationary estimators are also discussed in appendix A of \cite{Liu2018} where the authors hypothesize a potential connection to the more stringent concept of Rao-Blackwellization; our work shows that PDIS and SIS belong to the extended conditional Monte Carlo method and on which our conditional importance sampling framework is built.

The extended conditional Monte Carlo method is often attributed to \cite{Bratley1987}. \cite{Glasserman1993} studies the extended conditional Monte Carlo more generally under the name \textit{filtered Monte Carlo}. The sufficient condition for variance reduction in section \ref{sect:increasing_condition} is closely related to theorem 3.8 of \cite{Glasserman1993}, theorem 12 of \cite{Glynn1988}, the main theorem of \cite{Ross1988} on page 310 and exercise 2.6.3 of \cite{Bratley1987}. Our results in section \ref{sect:asymp_infinite} use elements of the proof techniques of \cite{glynn1996liapounov,glynn2019likelihood} but in the context of importance sampling for per-decision and stationary methods rather than for derivative estimation. The multiplicative structure of the importance sampling ratio in our setting renders impossible a direct application of those previous results to our setting.

Prior work has shown worst-case exponential lower bounds on the variance of IS and weighted IS \citep{jiang2015doubly,guo2017using}. On the upper bound side \citet[Lemma 1]{metelli2018policy} provides similar upper bound as our Theorem \ref{thm:is_variance}, but with one difference: our bound only use the second moment of one step ratio while theirs is the ratio of the whole trajectory. Additionally, we focus on the order of variance terms and derive lower bound and upper bound for the different estimators.. However, these results are derived with respect to specific MDPs while our Theorem \ref{thm:is_variance} provides general variance bounds. The recent work on stationary importance sampling~\cite{Hallak2017,Gelada2019,Liu2018} has prompted multiple further investigations. First, \cite{Xie2019} introduces the expression ``marginalized'' importance sampling to refer to the specific use of a stationary importance sampling estimator in conjunction with an estimate of an MDP model. This idea is related to both model-based reinforcement learning and the \textit{control variates} method for variance reduction \citep{Bratley1987,Lecuyer1994}; our work takes a different angle based on the extended conditional Monte Carlo. Our Corollary~\ref{cor:sis} about the variance of the stationary estimator matches their $O(T^2)$ dependency on the horizon but our result holds for general spaces and does not rely on having an estimate of the reward function.
 
\citet{voloshin2019empirical} also observed empirically that stationary importance sampling can yield a less accurate estimate than the crude importance sampling estimator or PDIS. Our analysis also considers how IS and PDIS might also vary in their accuracy, but focuses more broadly on building a theoretical understanding of those estimators and provide new variance bounds. Finally, parallel work by 
\citet{kallus2019efficiently} studies and analyzes incorporating control variates with stationary importance sampling by leveraging ideas of 
``double'' machine learning  \citep{kallus2019double,chernozhukov2016double} from semi-parametric inference. In contrast to that work, we provide a formal characterization of the variance of important sampling without control variates, and our results do not make the assumptions of a consistent value function estimator which is necessary for analysis by \citet{kallus2019efficiently}.

\section{Discussion}
Our analysis sheds new light on the commonly held belief that the stationary importance sampling estimators necessarily improve on their per-decision counterparts. As we show in section \ref{sect:counterexamples}, in short-horizon settings, there exist MDPs in which the stationary importance sampling estimator is provably worse than the per-decision one and both are worse than the crude importance sampling estimator. Furthermore, this increase in the variance occurs even if the stationary importance sampling ratio is given as oracle. To better understand this phenomenon, we establish a new connection between the per-decision and stationary estimators to the extended conditional Monte Carlo method. From this perspective, the potential lack of variance reduction is no longer surprising once we extend previous theoretical results from the simulation community to what we call ``conditional importance sampling''. This formalization help us derive sufficient conditions for variance reduction in theorems \ref{thm:episodic_pdis_variance} and \ref{thm:episodic_sis_variance} for the per-decision and stationary settings respectively. 

We then reconcile our theory with the known empirical success of stationary importance sampling through the theorems of section \ref{sect:asymp_infinite}. We show that under some assumptions, the intuition regarding PDIS and SIS does hold asymptotically and their variance can be polynomial in the horizon (corollary \ref{cor:pdis_variance_upperbound} and \ref{cor:sis} respectively) rather than exponential for the crude importance sampling estimator (theorem \ref{thm:is_variance}). Furthermore, we show through corollary \ref{cor:pdis_variance_lowerbound} and corollary \ref{cor:sis} that there exist conditions under which the variance of the stationary estimator is provably lower than the variance of the per-decision estimator. 

A natural next direction is exploring other statistics that better leverage the specific structure of an MDP, such as rewards, state abstractions, and find better conditional importance sampling estimator. Concurrent work \cite{rowland2019conditional} shows an interesting application of the reward conditioned estimator in online TD learning. However, in the batch setting, we prove that a reward conditioned estimator with a linear regression estimator yields an estimator that is equivalent to the vanilla IS estimator (See Appendix \ref{appendix:return_condition}). This interesting result highlights the subtle differences between online and batch settings. 
Exploring other statistics or lower bounds on this class of estimator is an interesting future direction.

In summary, the proposed framework of conditional importance sampling estimator both helps us understand existing estimators for batch off-policy policy evaluation and may lead to interesting future work by conditioning on different statistics. 

\section{Acknowledgements}
The authors would like to thank Peter Glynn and Pierre L'\'{E}cuyer for the useful discussions on variance reduction techniques and the conditional Monte Carlo method. This work was supported in part by an NSF CAREER award, an ONR Young Investigator Award and a research grant from Siemens.

\bibliographystyle{icml2020}
\bibliography{ref.bib}

\newpage
\onecolumn
\appendix
\section{Details of Couterexamples}
\label{appd:counterexamples_details}
In this section we provide details of computing the variance in Figure \ref{fig:counterexamples}. For each MDP, there are totally four possible trajectories (product of two actions and two steps), and the probabilities of them under behavior policy are all $1/4$. We list the return of different estimators for those four trajectories, then compute the variance of the estimators.  

\begin{table}[h]
\begin{center}
\begin{tabular}{ll|lll|lll|lll}
\toprule
\multicolumn{1}{c}{} &
\multicolumn{1}{c}{Probabilities} &
\multicolumn{3}{c}{Example \ref{fig:example1}} &
\multicolumn{3}{c}{Example \ref{fig:example2}} &
\multicolumn{3}{c}{Example \ref{fig:example3}} \\
            &of path        &IS     &PDIS   &SIS    &IS     &PDIS   &SIS    &IS     &PDIS   &SIS    \\
\midrule
 $a_1,a_1$  &0.25   &1.44   &1.2    &1.2    &0      &0      &0      &0      &0      &0  \\
 $a_1,a_2$  &0.25   &1.92   &2.16   &2.0    &1.44   &1.44   &1.2    &0.96   &0.96   &0.8\\
 $a_2,a_1$  &0.25   &0.96   &0.8    &0.8    &0.64   &0.8    &0.8    &0.96   &0.8    &0.8 \\
 $a_2,a_2$  &0.25   &1.28   &1.44   &1.6    &1.92   &1.76   &2.0    &1.28   &1.44   &1.6 \\
 \midrule 
 Expectation&   &1.4    &1.4    &1.4    &1      &1      &1      &0.8    &0.8    &0.8  \\
 Variance   &   &0.12   &0.2448 &0.2    &0.5424 &0.4528 &0.52   &0.2304 &0.2688 &0.32  \\
\bottomrule
\end{tabular}
\end{center}
\caption{Importance sampling returns and the variance.
See figure \ref{fig:counterexamples} for the problem structure.}
\end{table}

\section{Proof of Lemma \ref{lem:conditional_expectation_given_sa}}
\begin{proof}
In this proof, we use $\tau$ to denote the trajectory without reward: $\tau_{1:t} = \{s_k,a_k\}_{k=1}^t$. Since $\mathbb{E} (\rho_{1:t}|s_t,a_t) = \mathbb{E} (\rho_{1:t-1}|s_t,a_t)\rho_t$, we only need to prove that $\mathbb{E} (\rho_{1:t-1}|s_t,a_t) = \frac{d^\pi(s_t)}{d^\mu(s_t)}$.
\begin{align}
    \mathbb{E} (\rho_{1:t-1}|s_t,a_t) =& \int \prod_{k=1}^{t-1} \frac{\pi(s_k,a_k)}{\mu(s_k,a_k)} p_\mu(\tau_{1:t-1}|s_t,a_t) \text{d}\tau_{1:t-1} \\
    =& \int  \frac{ p_\pi(\tau_{1:t-1}) }{ p_\mu(\tau_{1:t-1}) } p_\mu(\tau_{1:t-1}|s_t,a_t) \text{d}\tau_{1:t-1} \\
    =& \int  \frac{ p_\pi(\tau_{1:t-1}) }{ p_\mu(\tau_{1:t-1}) } \frac{p_\mu(\tau_{1:t-1}) p(s_t|\tau_{1:t-1})\mu(a_t|s_t)}{p_\mu(s_t,a_t)} \text{d}\tau_{1:t-1} \\
    =& \int  \frac{ p_\pi(\tau_{1:t-1}) }{ p_\mu(\tau_{1:t-1}) } \frac{p_\mu(\tau_{1:t-1}) p(s_t|s_{t-1},a_{t-1})\mu(a_t|s_t)}{d^\mu_t(s_t)\mu(a_t|s_t)} \text{d}\tau_{1:t-1} \\
    =& \frac{1}{d^\mu_t(s_t)}\int p(s_t|s_{t-1},a_{t-1}) p_\pi(\tau_{1:t-1}) \text{d}\tau_{1:t-1} \\
    =& \frac{1}{d^\mu_t(s_t)}\int p(s_t|\tau_{1:t-1}) p_\pi(\tau_{1:t-1}) \text{d}\tau_{1:t-1} \\
    =& \frac{d^\pi_t(s_t)}{d^\mu_t(s_t)}
\end{align}
\end{proof}

\section{Proofs for Finite Horizon Case}

\subsection{Proof of Lemma \ref{lem:variance_of_sum}}
\begin{proof}
Since
$
    \mathbb{E}  \left( \sum_t \mathbb{E}(Y_t|X_t) \right)  = \mathbb{E} \left(\sum_t Y_t\right)
$,
we just need to compute the difference between the second moment of $\sum_t Y_t$ and $\sum_t \mathbb{E}(Y_t|X_t)$:
\begin{align}
    \mathbb{E}  \left( \sum_t \mathbb{E}(Y_t|X_t) \right)^2 =& \mathbb{E}  \left( \sum_t \left(\mathbb{E}(Y_t|X_t)\right)^2 + 2\sum_{t<k} \mathbb{E}(Y_t|X_t) \mathbb{E}(Y_k|X_k) \right) \\
    =& \sum_t \mathbb{E} \left(\mathbb{E}(Y_t|X_t)\right)^2 + 2\sum_{t<k} \mathbb{E}(\mathbb{E}(Y_t|X_t) \mathbb{E}(Y_k|X_k)) \\
    \le&  \sum_t \mathbb{E} \left(\mathbb{E}(Y_t^2|X_t)\right) + 2\sum_{t<k} \mathbb{E}(\mathbb{E}(Y_t|X_t) \mathbb{E}(Y_k|X_k)) \\
    =& \sum_t \mathbb{E} (Y_t^2)+ 2\sum_{t<k} \mathbb{E}(\mathbb{E}(Y_t|X_t) \mathbb{E}(Y_k|X_k))
\end{align}
\begin{align}
    \mathbb{E} \left(\sum_t Y_t\right) ^2 =& \mathbb{E} \left(\sum_t Y_t^2 + 2\sum_{t<k} Y_t Y_k  \right) \\
    =& \sum_t \mathbb{E} (Y_t^2) +2  \sum_{t<k} \mathbb{E}(Y_t Y_k) 
\end{align}
Thus we finished the proof by taking the difference between $\mathbb{E} \left(\sum_t Y_t\right) ^2$ and $\mathbb{E}  \left( \sum_t \mathbb{E}(Y_t|X_t) \right)^2$.
\end{proof}

\subsection{Proof of Theorem \ref{thm:episodic_pdis_variance}}
\begin{proof}
Let $\tau_{1:t}$ be the first $t$steps in a trajectory:  $(s_1,a_1,r_1,\dots,s_t,a_t,r_t)$, then $\rho_{1:t}r_t = \mathbb{E}(\rho_{1:T}r_t|\tau_{1:t})$. To prove the inequality between the variance of importance sampling and per decision importance sampling, we apply Lemma \ref{lem:variance_of_sum} to the variance, letting $Y_t = r_t \rho_{1:T}$ and $X_t = \tau_{1:t}$. Then it is sufficient to show that for any $1 \le t < k \le T$, 
\begin{align}
    \mathbb{E}(r_t r_k \rho_{1:T} \rho_{1:T}) = \mathbb{E}(Y_t Y_k) \ge  \mathbb{E}(\mathbb{E}(Y_t|X_t) \mathbb{E}(Y_k|X_k)) = \mathbb{E}(r_t r_k \rho_{1:t} \rho_{1:k})
\end{align}
To prove that, it is sufficient to show $\mathbb{E}(r_t r_k \rho_{1:T} \rho_{1:T}|\tau_{1:t}) \ge \mathbb{E}(r_t r_k \rho_{1:t} \rho_{1:k}|\tau_{1:t})$. Since
\begin{align}
    \mathbb{E}\left(r_t r_k \rho_{1:t} \rho_{1:k}|\tau_{1:t}\right) &= r_t \rho_{1:t}^2 \mathbb{E}\left( r_k \rho_{t+1:k}|\tau_{1:t}\right) \\
    &= r_t \rho_{1:t}^2 \mathbb{E}\left( r_k \rho_{t+1:T}|\tau_{1:t}\right) \\
    &= r_t \rho_{1:t}^2 \mathbb{E}\left( r_k \rho_{t+1:T}|\tau_{1:t}\right) \mathbb{E}\left( \rho_{t+1:T}|\tau_{1:t}\right)\\
\end{align}
Given $\tau_{1:t}$, $r_k$ and $\rho_{t+1:T}$ can be viewed as $r_{k-t+1}$ and $\rho_{1:T-t+1}$ on a new trajectory. Then according to the statement of theorem, $r_{k-t+1}\rho_{1:T-t+1}$ and $\rho_{1:T-t+1}$ are are positively correlated. Now we can upper bound $\mathbb{E}\left(r_t r_k \rho_{1:t} \rho_{1:k}|\tau_{1:t}\right)$ by:
\begin{align}
    r_t \rho_{1:t}^2 \mathbb{E}\left( r_k \rho_{t+1:T}|\tau_{1:t}\right) \mathbb{E}\left( \rho_{t+1:T}|\tau_{1:t}\right) &\le r_t \rho_{1:t}^2 \mathbb{E}\left( r_k \rho_{t+1:T} \rho_{t+1:T}|\tau_{1:t}\right) \\
    &= \mathbb{E}\left( r_k r_t \rho_{1:T} \rho_{1:T}|\tau_{1:t}\right) 
\end{align}
This implies $\mathbb{E}(r_t r_k \rho_{1:T} \rho_{1:T}) \ge \mathbb{E}(r_t r_k \rho_{1:t} \rho_{1:k})$ by taking expectation over $\tau_{1:t}$, and finish the proof.
\end{proof}

\subsection{Proof of Theorem \ref{thm:episodic_sis_variance}}
\begin{proof}
Using lemma \ref{lem:variance_of_sum} by $Y_t = \rho_{1:t}r_t$ and $X_t = s_t,a_t,r_t$ , we have that the variance of \sis~ is smaller than the variance of \pdis~ if for any $t<k$: 
\begin{align}
    \mathbb{E}\left[ \rho_{1:t} \rho_{0:k} r_t r_k \right] \ge& \mathbb{E}\left[ \mathbb{E}(\rho_{1:t}|s_t,a_t) \mathbb{E}(\rho_{0:k}|s_k,a_k) r_t r_k \right] \\
    =& \mathbb{E}\left[ \frac{d^\pi_t(s,a)}{d^\mu_t(s,a)} \frac{d^\pi_k(s,a)}{d^\mu_k(s,a)} r_t r_k \right]
\end{align}
The second line follows from Lemma \ref{lem:conditional_expectation_given_sa} to simplify $\mathbb{E}(\rho_{0:t}|s_t,a_t)$. 
To show that, we will transform the above equation into a an expression about two covariances. To proceed
we subtracting $\mathbb{E}(\rho_{1:t} r_t) \mathbb{E}(\rho_{1:k} r_k)$ from both sides, and note that the resulting left hand side is simply the covariance:
\begin{align}
    \cov \left[ \rho_{1:t} r_t,  \rho_{0:k}r_k \right] \nonumber =& \mathbb{E}\left[ \rho_{1:t} \rho_{1:k} r_t r_k \right] - \mathbb{E}(\rho_{1:t} r_t) \mathbb{E}(\rho_{1:k} r_k) \nonumber \\
    \ge& \mathbb{E}\left[ \frac{d^\pi_t(s,a)}{d^\mu_t(s,a)} \frac{d^\pi_k(s,a)}{d^\mu_k(s,a)} r_t r_k \right] - \mathbb{E}(\rho_{1:t} r_t) \mathbb{E}(\rho_{1:k} r_k)
    \label{eq:episodic_sis_proof_line}
\end{align}
We now expand the second term in the right hand side
\begin{align}
    \mathbb{E}(\rho_{1:t} r_t) \mathbb{E}(\rho_{1:k} r_k) = & \mathbb{E}(r_t\mathbb{E}(\rho_{1:t}|s_t,a_t))\mathbb{E}( r_k\mathbb{E}(\rho_{1:k}|s_k,a_k)) \\
    = & \mathbb{E}\left[ \frac{d^\pi_t(s,a)}{d^\mu_t(s,a)}  r_t  \right] \mathbb{E}\left[\frac{d^\pi_k(s,a)}{d^\mu_k(s,a)}  r_k \right] 
\end{align}
This shows that both sides of 
\ref{eq:episodic_sis_proof_line} are covariances. The result then follows under the assumption of the proof. 
\end{proof}

\section{Proofs for infinite horizon case}
\subsection{Proof of Theorem \ref{thm:likelihood_ratio_distr}}
\begin{proof}
We can write the log of likelihood ratio as sum of random variables on a Markov chain,
\begin{align}
    \log \rho_{1:T} = \sum_{t=1}^T \log \rho_t = \sum_{t=1}^T \log \left( \frac{\pi(a_t|s_t)}{\mu(a_t|s_t)} \right)
\end{align}
By the strong law of large number on Markov chain \cite{breiman1960strong}:
\begin{align}
 \frac{1}{T} \log \rho_{1:T} = \frac{1}{T} \sum_{i=1}^T \log \left( \frac{\pi(a_i|s_i)}{\mu(a_i|s_i)} \right) \rightarrow_{a.s.} \mathbb{E}_{d^\mu} \log \left( \frac{\pi(a_i|s_i)}{\mu(a_i|s_i)} \right) = -c
\end{align}
If $\pi \neq \mu$, the strict concavity of $\log$ function implies that:
\begin{align}
    c = \mathbb{E}_{d^\mu} \log \left( \frac{\pi(a|s)}{\mu(a|s)} \right) < \log \mathbb{E}_{d^\mu} \left( \frac{\pi(a|s)}{\mu(a|s)} \right)  = 0
\end{align}
Thus $\frac{1}{T} \log \rho_{1:T} \to_{a.s.} c $ and $\rho_{1:T}^{1/T} \to_{a.s.} e^{-c} $. Since $r_t \le 1$, $| \rho_{1:T} \sum_{t=1}^T \gamma^{t-1} r_t |^{1/T} \le  \rho_{1:T}^{1/T} T^{1/T}$. Since $T^{1/T} \to 1$, $\overline{\lim}_{T \to \infty} | \rho_{1:T} \sum_{t=1}^T \gamma^{t-1} r_t |^{1/T} < e^{-c}$.
\end{proof}

\subsection{Proof of Corollary \ref{cor:likelihood_ratio_distr}}
\begin{proof}
$\rho_{1:T} \rightarrow_{a.s.} 0$ directly follows from $ \rho_{1:T}^{1/T} \rightarrow_{a.s.} e^{-c}$ in Theorem \ref{thm:likelihood_ratio_distr}. For $\rho_{1:T} \sum_{t=1}^T \gamma^{t-1}r_t$, if there exist $\epsilon > 0$ such that $\rho_{1:T} \sum_{t=1}^T \gamma^{t-1}r_t > \epsilon$ for any T, then:
\begin{align*}
   \overline{\lim}_{T \to \infty} \left| \rho_{1:T} \sum_{t=1}^T \gamma^{t-1}r_t \right|^{1/T} \ge \overline{\lim}_{T \to \infty} \epsilon^{1/T} = 1
\end{align*}
This contradicts $e^{-c} > \overline{\lim}_{T}| \rho_{1:T} \sum_{t=1}^T \gamma^{t-1}r_t |^{1/T}$ So $\overline{\lim}_{T} \rho_{1:T} \sum_{t=1}^T \gamma^{t-1}r_t \le 0$, which implies that $\rho_{1:T} \sum_{t=1}^T \gamma^{t-1}r_t \to_{a.s.} 0$
\end{proof}

\subsection{Proof of Lemma \ref{lem:likelihood_ratio_martingale}}
\begin{proof}
Let $f(s,a) = \log \frac{\pi(s,a)}{\mu(s,a)}$. According to Assumption \ref{assum:bounded_ratio}, $|f(s,a)| < \infty$. Since $B(s,a) \ge 1$, $\frac{|f(s,a)|}{\sqrt{B(s,a)}} < \infty$. Since $f^2$ and $B$ are both finite, $\mathbb{E}_{d^\mu}f^2 < \infty$ and $\mathbb{E}_{d^\mu}B < \infty$. Now we satisfy the condition of Lemma 3 in \cite{glynn2019likelihood}: in the proof of Lemma 3 in in \cite{glynn2019likelihood} they used their Assumption i) Harris Chain, which is our Assumption \ref{assum:harris}, their Assumption vii) $||f||_sqrt{V}$ bounded (whic is satisfied by our bound on B in Assumption \ref{assum:drift}), which is explained by f is bounded and $\sqrt{B} >= 1$, and finally their assumption iv), which is our assumption \ref{assum:drift}. The only difference is we assume a ``petite'' K which is a slight generalization of the ``small'' set K (See discussion in \citep[Section 5]{meyn2012markov}). The proof in Meyn and Glynn 1996 also used petite (which is the part where Glynn and Olvera-Cravioto need assumption iv)). This assumption (drift condition) is often necessary for quantitative analysis of general state Markov Chains. The geometric ergodicity for general state MC is also defined with a petite/small set. By Thm 15.0.1 in Meyn and Tweedie the drift property is equivalent to geometric ergodicity. According to Lemma 3 in \cite{glynn2019likelihood}, whose proof is similar with Theorem 2.3 in \cite{glynn1996liapounov}, we have that there exist a solution $\hat{f}$ to the following Poisson's equation:
\begin{align}
    \hat{f}(s,a) - \mathbb{E}_{\cdot|s,a}\hat{f}(s',a') = f(s,a) - \mathbb{E}_{d^\mu} f(s,a)
\end{align}
satisfying $|\hat{f}(s,a)| < c_1\sqrt{B(s,a)}$ for some constant $c_1$. Following from the Poisson's equation we have:
\begin{align}
    \log \rho_{1:T} + Tc =& \sum_{t=1}^T \left(f(s_t,a_t) - \mathbb{E}_{d^\mu} f(s,a)\right) \\
    =&\sum_{t=1}^T \left( \hat{f}(s_t,a_t) - \mathbb{E}_{s',a'|s_t,a_t}\hat{f}(s',a') \right) \\
    =& \hat{f}(s_1,a_1) - \hat{f}(s_{T+1},a_{T+1})  + \sum_{t=2}^{T+1} \left( \hat{f}(s_t,a_t) - \mathbb{E}_{s',a'|s_{t-1},a_{t-1}}\hat{f}(s',a') \right) 
\end{align}
$\left( \hat{f}(s_t,a_t) - \mathbb{E}_{s',a'|s_{t-1},a_{t-1}}\hat{f}(s',a') \right)$ are martingale differences. The absolute value of difference is upper bounded by $2\|\hat{f}\|_\infty \le 2c_1\sqrt{\|B\|_\infty}$.
\end{proof}

\subsection{Proof of Theorem \ref{thm:is_variance}}
\begin{lemma}
\label{lem:prod_of_ratio} If $\mathbb{E}_{\mu}[\rho^2|s] \le M_{\rho}^2$ for any $s$, $\mathbb{E}\left[ \rho_{0:k}^2 \right] \le  M_{\rho}^{2k} $
\end{lemma}
\begin{proof}
\begin{align}
    \mathbb{E}\left[ \rho_{0:t}^2 \right] &= \mathbb{E} \left[ \prod_{i=1}^{k}\rho_i^2 \right]  \\
    &= \mathbb{E} \left[\left( \prod_{i=1}^{k-1}\rho_i^2\right) \mathbb{E}_{s_k,a_k} \left[ \rho_k^2 | s_1,a_1,s_2,\dots,s_{k-1},a_{k-1} \right] \right] \\
    &\le \mathbb{E} \left[\left( \prod_{i=1}^{k-1}\rho_i^2 \right) M_{\rho}^2 \right] \\
    &= M_{\rho}^2 \mathbb{E} \left[ \prod_{i=1}^{k-1}\rho_i^2 \right] \\
    & \dots\\
    &= M_{\rho}^{2k}
\end{align}
\end{proof}

\begin{proof}
Define $Y = \rho_{1:T} \sum_{t=1}^T \gamma^{t-1} r_t$ and $Z = \mathds{1}(Y > v^{\pi}/2)$, then $v^{\pi} = \mathbb{E}(Y)$. By the law of total variance, 
\begin{align}
    \var(Y) &= \var(\mathbb{E}(Y|Z)) + \mathbb{E}(\var(Y|Z)) \\
    &\ge \var(\mathbb{E}(Y|Z)) \\
    &= \mathbb{E}(\mathbb{E}(Y|Z))^2 - (v^{\pi})^2 \\
    &\ge \Pr(Y > v^{\pi}/2) (\mathbb{E}(Y|Y>v^{\pi}/2))^2 - (v^{\pi})^2 
    \label{eq:variance_lb}
\end{align}
Now we are going to lower bound $\mathbb{E}(Y|Y>v^\pi/2)$. We can rewrite $\mathbb{E}(Y)=v^\pi$ as:
\begin{align}
    v^{\pi} =& \mathbb{E}(Y) = \mathbb{E}(\mathbb{E}(Y|Z))\\
    =& \Pr(Y > v^{\pi}/2)\mathbb{E}(Y|Y>v^{\pi}/2)  +  \Pr(Y \le v^{\pi}/2)\mathbb{E}(Y|Y\le v^{\pi}/2) \\
    \le&  \Pr(Y > v^{\pi}/2)\mathbb{E}(Y|Y>v^{\pi}/2) + 1 \times v^{\pi}/2
\end{align}
So $\mathbb{E}(Y|Y>v^{\pi}/2) \ge \frac{v^{\pi}}{2\Pr(Y > v^{\pi}/2)}$. Substitute this into the RHS of Equation \ref{eq:variance_lb}:
\begin{align}
     \var(Y) \ge \frac{(v^{\pi})^2}{4\Pr(Y > v^{\pi}/2)} - (v^{\pi})^2
\end{align}
Now we are going to upper bound $\Pr(Y > v^{\pi}/2)$. Recall that we define $c = \mathbb{E}_{d^\mu}D_{\text{KL}}(\mu||\pi) =- \mathbb{E}_{d^\mu} \log \left( \frac{\pi(a|s)}{\mu(a|s)} \right)$. Now we define $c(T) = - \mathbb{E}_{d^\mu_{1:T}} \log \left( \frac{\pi(a|s)}{\mu(a|s)} \right) = - \frac{1}{T}\mathbb{E}_{\mu} [\log \rho_{1:T}]$.
\begin{align}
    & \Pr(Y > v^\pi/2) \\
    =& \Pr( \rho_{1:T} \sum_{t=1}^T \gamma^{t-1} r_t  > v^\pi/2) \le \Pr( \rho_{1:T} T > v^\pi/2) \\
    =& \Pr\left( \rho_{1:T} > \frac{v^\pi}{2T}\right) \\
    =& \Pr\left( \log \rho_{1:T} > \log v^\pi -\log(2T)\right) \\
    =& \Pr\left( \frac{ \log \rho_{1:T}}{T} > \frac{\log v^\pi-\log(2T)}{T} \right) \\
    =& \Pr\left( \frac{\log \rho_{1:T}}{T} + c + \frac{\hat{f}(s_{T+1},a_{T+1}) - \hat{f}(s_1,a_1)}{T} > c  + \frac{\log v^\pi - \log(2T) + \hat{f}(s_{T+1},a_{T+1}) - \hat{f}(s_1,a_1)}{T}\right)
\end{align}
Since $\log v^\pi$ is a constant, $\hat{f}(s_{T+1},a_{T+1}) - \hat{f}(s_1,a_1)$ could be upper bounded by constant $2c_1\sqrt{\|B\|_\infty}$, and $\lim_{T\to \infty} \frac{\log(2T)}{T} = 0 $, we know that $\lim_{T\to \infty}\frac{\log v^\pi - \log(2T) + \hat{f}(s_{T+1},a_{T+1}) - \hat{f}(s_1,a_1)}{T} = 0$. So there exists a constant $T_0 > 0$ such that for all $T > T_0$, 
\begin{align*}
   \frac{\log v^\pi - \log(2T) + \hat{f}(s_{T+1},a_{T+1}) - \hat{f}(s_1,a_1)}{T} > -\frac{c}{2}
\end{align*}
Therefore for all $T>T_0$:
\begin{align*}
    \Pr(Y > v^\pi/2) \le \Pr\left( \frac{\log \rho_{1:T}}{T} + c + \frac{\hat{f}(s_{T+1},a_{T+1}) - \hat{f}(s_1,a_1)}{T} > c/2 \right)
\end{align*}
According to Lemma \ref{lem:likelihood_ratio_martingale}, and Azuma's inequality\cite{azuma1967weighted}, we have:
\begin{align*}
    \Pr(Y > v^\pi/2) \le \exp \left( \frac{-Tc^2}{8 c_1^2 \| B \|_\infty} \right)
\end{align*}
Thus we can lower bound the variance of importance sampling estimator $Y$:
\begin{align}
    \var(Y) \ge \frac{(v^\pi)^2}{4} \exp \left( \frac{Tc^2}{8 c_1^2 \| B \|_\infty} \right) - (v^\pi)^2
\end{align}
If the one step likelihood ratio is upper bounded by $\rhomax$, then the variance of importance sampling estimator can be upper bounded by:
\begin{align}
    \var(\is) = \mathbb{E} [Y^2] - (v^\pi)^2 &= \mathbb{E}\left[ \rho_{0:T}^2 \left( \sum_{t=1}^T \gamma^{t-1}r_t \right)^2 \right] - (v^\pi)^2 \\ 
    &\le T^2 \mathbb{E}\left[ \rho_{0:T}^2 \right] - (v^\pi)^2 \\
    &\le T^2\rhomax^{2T} - (v^\pi)^2
\end{align}
Following from lemma \ref{lem:prod_of_ratio}, the variance term can also be upper bounded by:
\begin{align}
    \var(\is) = \mathbb{E} [Y^2] - (v^\pi)^2 &= \mathbb{E}\left[ \rho_{0:T}^2 \left( \sum_{t=1}^T \gamma^{t-1}r_t \right)^2 \right] - (v^\pi)^2 \\ 
    &\le T^2 \mathbb{E}\left[ \rho_{0:T}^2 \right] - (v^\pi)^2 \\
    &\le T^2 M_{\rho}^{2T} - (v^\pi)^2
\end{align}
\end{proof}

\subsection{Proof of Theorem \ref{thm:pdis_variance}}
\begin{proof}
Let $Y_t =  \rho_{1:t} \gamma^{t-1} r_t$. For the upper bound:
\begin{align}
    \var(\pdis) =& \mathbb{E} \left( \left(\sum_{t=1}^T Y_t\right)^2 \right) - (v^{\pi})^2 \\
    \le&  \mathbb{E} \left( T\sum_{t=1}^T Y_t^2 \right) - (v^{\pi})^2 \\
    =& T\sum_{t=1}^T \mathbb{E}(Y_t^2) - (v^{\pi})^2 \\ 
    =& T\sum_{t=1}^T \mathbb{E}(\rho_{0:t}^2 \gamma^{2t-2} r_t^2) - (v^{\pi})^2 \\
    \le& T\sum_{t=1}^T \rhomax^{2t} \gamma^{2t-2}\mathbb{E}_{\mu}[(r_t)^2]  - (v^\pi)^2
\end{align}
Or it can also be bounded as:
\begin{align}
    \var(\pdis) \le& T\sum_{t=1}^T \mathbb{E}(\rho_{0:t}^2 \gamma^{2t-2} r_t^2) - (v^{\pi})^2 \\
    =& T\sum_{t=1}^T \gamma^{2t-2} \mathbb{E}(\rho_{0:t}^2 ) - (v^{\pi})^2 \\
    \le& T\sum_{t=1}^T \gamma^{2t-2} 
    M_\rho^{2t} - (v^{\pi})^2 
\end{align}
The last step follows from lemma \ref{lem:prod_of_ratio}. For the lower bound, we notice that $Y_t \ge 0$ for any $t$, then:
\begin{align}
    \mathbb{E} \left( \left(\sum_{t=1}^T Y_t\right)^2 \right) \ge \mathbb{E} \left( \sum_{t=0}^T Y_t^2 \right) = \sum_{t=1}^T \mathbb{E} (Y_t^2)
\end{align}
For each $t$, we will follow a similar proof as how to lower bound part in Theorem \ref{thm:is_variance}:
\begin{align}
    \mathbb{E} (Y_t^2) =& \mathbb{E} \left( \mathbb{E} (Y_t^2|\mathds{1}(Y_t > \gamma^{t-1}\mathbb{E}_{\pi}(r_t)/2)) \right) \\
    \ge& \mathbb{E} \left( \mathbb{E} (Y_t|\mathds{1}(Y_t > \gamma^{t-1}\mathbb{E}_{\pi}(r_t)/2)) \right)^2 \\
    \ge& \Pr(Y_t > \gamma^{t-1}\mathbb{E}_{\pi}(r_t)/2) \left( \mathbb{E} (Y_t|Y_t > \gamma^{t-1}\mathbb{E}_{\pi}(r_t)/2)\right)^2
\end{align}
Notice that $\mathbb{E} (Y_t) = \gamma^{t-1}\mathbb{E}_{\pi}(r_t)$, 
\begin{align}
    & \gamma^{t-1}\mathbb{E}_{\pi}(r_t) = \mathbb{E}(Y_t) \\
    =&\Pr(Y_t > \gamma^{t-1}\mathbb{E}_{\pi}(r_t)/2) \mathbb{E} (Y_t|Y_t > \gamma^{t-1}\mathbb{E}_{\pi}(r_t)/2)  + \Pr(Y_t \le \gamma^{t-1}\mathbb{E}_{\pi}(r_t)/2)\mathbb{E} (Y_t|Y_t \le \gamma^{t-1}\mathbb{E}_{\pi}(r_t)/2) \\
    \le& \Pr(Y_t > \gamma^{t-1}\mathbb{E}_{\pi}(r_t)/2)\mathbb{E} (Y_t|Y_t > \gamma^{t-1}\mathbb{E}_{\pi}(r_t)/2) + \gamma^{t-1}\mathbb{E}_{\pi}(r_t)/2
\end{align}
So we can lower bound the $\mathbb{E} (Y_t^2)$:
\begin{align}
     \mathbb{E} (Y_t|Y_t > \gamma^{t-1}\mathbb{E}_{\pi}(r_t)/2) \ge \frac{\gamma^{t-1}\mathbb{E}_{\pi}(r_t)}{2\Pr(Y_t > \gamma^{t-1}\mathbb{E}_{\pi}(r_t)/2)}\\
     \mathbb{E} (Y_t^2) \ge \frac{\gamma^{2t-2}\left(\mathbb{E}_{\pi}(r_t)\right)^2}{4\Pr(Y_t > \gamma^{t-1}\mathbb{E}_{\pi}(r_t)/2)}
\end{align}
Now we are going to upper bound the tail probability $\Pr(Y_t > \gamma^{t-1}\mathbb{E}_{\pi}(r_t)/2)$:
\begin{align}
    &\Pr\left(Y_t|Y_t > \frac{\gamma^{t-1}\mathbb{E}_{\pi}(r_t)}{2}\right) \\
    =& \Pr\left( \rho_{1:t} \gamma^{t-1} r_t  > \frac{\gamma^{t-1}\mathbb{E}_{\pi}(r_t)}{2} \right) \\
    \le& \Pr\left( \rho_{1:t}  > \frac{\mathbb{E}_{\pi}(r_t)}{2} \right) \\
    =& \Pr\left( \log \rho_{1:t} > \log \mathbb{E}_{\pi}(r_t) -\log2\right) \\
    =& \Pr\left( \frac{1}{t} \log \rho_{1:t} > \frac{\mathbb{E}_{\pi}(r_t) -\log2}{t} \right) \\
    =& \Pr\left( \frac{1}{t} \log \rho_{1:t} + c + \frac{\hat{f}(s_{t+1},a_{t+1}) - \hat{f}(s_1,a_1)}{T} > c + \frac{\mathbb{E}_{\pi}(r_t) -\log2 + \hat{f}(s_{t+1},a_{t+1}) - \hat{f}(s_1,a_1)}{t}\right)
\end{align}
Since $|\mathbb{E}_{\pi}(r_t) -\log2 + \hat{f}(s_{t+1},a_{t+1}) - \hat{f}(s_1,a_1)|$ is bounded, there exist some $T_0 > 0$ such that if $t>T_0$, we can lower bound the right hand side in the probability by $c/2$. Then for $t>T_0$, by Azuma's inequality \citep{azuma1967weighted},
\begin{align}
    \Pr\left(Y_t|Y_t > \frac{\gamma^{t-1}\mathbb{E}_{\pi}(r_t)}{2}\right)  \le& \Pr\left(\frac{\log \rho_{1:t}}{t} + c + \frac{\hat{f}(s_{t+1},a_{t+1}) - \hat{f}(s_1,a_1)}{t} > \frac{c}{2}\right) \\
    \le& \exp \left( \frac{-tc^2}{8c_1^2\|B\|_\infty}\right)
\end{align}
So we have that for $t>T_0$:
$$ \mathbb{E}(Y_t^2) \ge \frac{\gamma^{2t-2}\mathbb{E}_{\pi}(r_t)}{4} \exp \left( \frac{tc^2}{8c_1^2\|B\|_\infty}\right)$$ For $0<t \le T_0$, $ \mathbb{E}(Y_t^2)  \ge 0$ completes the proof.
\end{proof}

\subsection{Proof of Corollary \ref{cor:pdis_variance_lowerbound}}
\begin{proof}
First, $\gamma \ge \exp \left( \frac{-c^2}{16c_1^2\|B\|_\infty}\right)$ indicate $\left(\frac{c^2}{8c_1^2 \| B \|_\infty}  + 2\log\gamma \right) > 0$. This is necessary for the second condition to hold since $r_t < 1$. The second condition $\mathbb{E}_\pi(r_t) = \Omega\left(\exp \left( \frac{-tc^2}{8c_1^2 \| B \|_\infty} - 2t\log\gamma + \epsilon t/2 \right)\right)$ implies that there exist a $T_1 > 0$ and a constant $C>0$ such that $(\mathbb{E}_\pi(r_t))^2 \ge C\left(\exp \left( \frac{-tc^2}{8c_1^2 \| B \|_\infty} - 2t\log\gamma + \epsilon t \right)\right)$, for any $t>T_1$. Then let $T>\max\{T_1,T_0\}$, where $T_0$ is the constant in Theorem \ref{thm:pdis_variance}:
\begin{align}
    \var(\sum_{t=T_0}^T \rho_{1_t} \gamma^{t-1}r_t) \ge& \sum_{t=1}^T  \frac{\gamma^{2t-2}(\mathbb{E}_{\pi}(r_t))^2}{4} \exp \left( \frac{tc^2}{8c_1^2 \| B \|_\infty} \right) - (v^\pi)^2 \\
    \ge& \frac{\gamma^{2T-2}(\mathbb{E}_{\pi}(r_T))^2}{4} \exp \left( \frac{Tc^2}{8c_1^2 \| B \|_\infty} \right) - (v^\pi)^2 \\
    \ge& \frac{\gamma^{-2}C}{4} \exp(\epsilon T) - (v^\pi)^2 = \Omega(\exp{\epsilon T})
\end{align}
\end{proof}

\subsection{Proof of Corollary \ref{cor:pdis_variance_upperbound}}
\begin{proof}
If $\rhomax \gamma \le 1$, $\rhomax^t \gamma^{t-1} \mathbb{E}_{\pi}(r_t) \le 1/\gamma$ for any $t$ since $r_t \in [0,1]$. If $\rhomax \gamma \lim \left(\mathbb{E}_{\mu}(r_T)\right)^{1/T} < 1$, let $\delta = 1 - \rhomax \gamma \lim \left(\mathbb{E}_{\mu}(r_T)\right)^{1/T} > 0$. There exist a $T_0>0$ such that for all $t>T_0$, $\rhomax \gamma (\mathbb{E}_{\pi}(r_t))^{1/t} \le \rhomax \gamma ( \lim \left(\mathbb{E}_{\mu}(r_T)\right)^{1/T} + \delta/2(\rhomax \gamma)) = 1-\delta/2 < 1$. Therefore in both case, for all $T > T_0$, $\rhomax^t \gamma^{t-1} \mathbb{E}_{\mu}(r_T)\le 1/\gamma$:

\begin{align}
    \var(\sum_{t=1}^T \rho_{1:t} \gamma^{t-1} r_t) \le& T\sum_{t=1}^T \rhomax^t \gamma^{t-1}\mathbb{E}_{\mu}(r_T) \le T\sum_{t=1}^{T_0}  \rhomax^t \gamma^{t-1}\mathbb{E}_{\mu}(r_T)  + T\sum_{t=T_0+1}^T \rhomax^t \gamma^{t-1}\mathbb{E}_{\mu}(r_T) \\
    \le& T T_0 \frac{\rhomax^{T_0}-1}{\rhomax-1} + 2T^2\frac{1}{\gamma}
\end{align}
Since $T_0$ is a constant, the variance is $O(T^2)$.
\end{proof}

\subsection{Proof of Theorem \ref{thm:sis_variance}}
\begin{proof}
\begin{align}
    &\var \left(\sum_{t=1}^T \frac{d^{\pi}_t(s_t,a_t)}{d^{\mu}_t(s_t,a_t)} \gamma^{t-1} r_t \right) \\ 
    =& \sum_{t=1}^T \var \left(\frac{d^{\pi}_t(s_t,a_t)}{d^{\mu}_t(s_t,a_t)} \gamma^{t-1} r_t \right) + 2 \sum_{t<k} \text{Cov} \left( \frac{d^{\pi}_t(s_t,a_t)}{d^{\mu}_t(s_t,a_t)} \gamma^{t-1} r_t, \frac{d^{\pi}_k(s_k,a_k)}{d^{\mu}_k(s_k,a_k)} \gamma^{k-1} r_k  \right) \\
    \le& \sum_{t=1}^T \var \left(\frac{d^{\pi}_t(s_t,a_t)}{d^{\mu}_t(s_t,a_t)} \gamma^{t-1} r_t \right) + \sum_{t<k} 2\sqrt{ \var \left( \frac{d^{\pi}_t(s_t,a_t)}{d^{\mu}_t(s_t,a_t)} \gamma^{t-1} r_t\right) \var\left( \frac{d^{\pi}_k(s_k,a_k)}{d^{\mu}_k(s_k,a_k)} \gamma^{k-1} r_k  \right)}\\
    \le& \sum_{t=1}^T \var \left(\frac{d^{\pi}_t(s_t,a_t)}{d^{\mu}_t(s_t,a_t)} \gamma^{t-1} r_t \right) + \sum_{t<k} \left( \var \left( \frac{d^{\pi}_t(s_t,a_t)}{d^{\mu}_t(s_t,a_t)} \gamma^{t-1}r_t\right) + \var\left( \frac{d^{\pi}_k(s_k,a_k)}{d^{\mu}_k(s_k,a_k)} \gamma^{k-1}r_k  \right) \right)\\
    =& T\sum_{t=1}^T \gamma^{2t-2} \var \left(\frac{d^{\pi}_t(s_t,a_t)}{d^{\mu}_t(s_t,a_t)} r_t \right)\\
    \le& T\sum_{t=1}^T \gamma^{2t-2} \var \left(\frac{d^{\pi}_t(s_t,a_t)}{d^{\mu}_t(s_t,a_t)} \right)\\
    =& T \sum_{t=1}^T \gamma^{2t-2} \left( \mathbb{E}\left( \frac{d^{\pi}_t(s_t,a_t)}{d^{\mu}_t(s_t,a_t)} \right)^2 - 1 \right)
\end{align}
\end{proof}

\subsection{Proof of Corollary \ref{cor:sis}}
\setcounter{lemma}{3}
\begin{lemma}
\label{lem:converge_stationary_ratio}
If $d^\mu_t(s_t)$ and $d^\pi_t(s_t)$ are asymptotically equi-continuous, $\frac{d^\pi(s)}{d^\mu(s)} \le \wmax$, and $\frac{\pi(a|s)}{\mu(a|s)} \le \rhomax$, then,
$$\lim_t \mathbb{E}_{s_t,a_t \sim d^\mu_t}\left( \frac{d^{\pi}_t(s_t,a_t)}{d^{\mu}_t(s_t,a_t)} \right)^2 = \mathbb{E}_{s,a \sim d^\mu}\left( \frac{d^{\pi}(s,a)}{d^{\mu}(s,a)} \right)^2$$ 
\end{lemma}
\begin{proof}
According to the law of large number on Markov chain \cite{breiman1960strong}, the distribution of $d^\mu_t$ converge to the stationary distribution $d^\mu$ in distribution. By the Lemma 1 in \cite{boos1985converse}, $d^\mu_t(s,a)$ converge to $d^\mu(s,a)$ pointwisely, $d^\pi_t(s,a)$ converge to $d^\pi(s,a)$ pointwisely. So $\frac{d^\pi_t(s)}{d^\mu_t(s)}$ converge to $\frac{d^\pi(s)}{d^\mu(s)}$ pointwisely.

\begin{align}
    &\mathbb{E}_{s_t,a_t \sim d^\mu_t}\left( \frac{d^{\pi}_t(s_t,a_t)}{d^{\mu}_t(s_t,a_t)} \right)^2 - \mathbb{E}_{s,a \sim d^\mu}\left( \frac{d^{\pi}(s,a)}{d^{\mu}(s,a)} \right)^2 \\
    =& \int_{s,a}  \frac{\left( d^{\pi}_t(s,a) \right)^2}{d^{\mu}_t(s,a)} \dsda - \int_{s,a}  \frac{\left( d^{\pi}(s,a) \right)^2}{d^{\mu}(s,a)} \dsda \\
    =& \int_{s,a}  \frac{\left( d^{\pi}_t(s) \right)^2}{d^{\mu}_t(s)} \frac{(\pi(a|s))^2}{\mu(a|s)}  - \frac{\left( d^{\pi}(s) \right)^2}{d^{\mu}(s)}  \frac{(\pi(a|s))^2}{\mu(a|s)}  \dsda \\
    \le& \rhomax \int_{s,a}  \left| \frac{\left( d^{\pi}_t(s) \right)^2}{d^{\mu}_t(s)}  - \frac{\left( d^{\pi}(s) \right)^2}{d^{\mu}(s)}  \right| \dsda \\
    \le& \rhomax \int_{s,a} \left| \frac{d^{\pi}(s)\left( d^{\pi}_t(s) - d^{\pi}(s) \right) }{d^{\mu}(s)}  + d^{\pi}_t(s) \frac{ d^{\pi}_t(s) }{d^{\mu}_t(s)} - d^{\pi}_t(s) \frac{ d^{\pi}(s) }{d^{\mu}(s)} \right| \dsda \\
    \le& \rhomax \int_{s,a} \left| \frac{d^{\pi}(s)\left( d^{\pi}_t(s) - d^{\pi}(s) \right) }{d^{\mu}(s)} \right|  + d^{\pi}_t(s) \left|  \frac{ d^{\pi}_t(s) }{d^{\mu}_t(s)} - \frac{ d^{\pi}(s) }{d^{\mu}(s)} \right| \dsda \\
    \le& \rhomax\wmax \TV (d^{\pi}_t, d^{\pi}) + \rhomax \int_{s,a} \left| \frac{ d^{\pi}_t(s) }{d^{\mu}_t(s)} - \frac{ d^{\pi}(s) }{d^{\mu}(s)} \right|\dsda 
\end{align}
By the law of large number on Markov chain \citep{breiman1960strong}, $\TV (d^{\pi}_t, d^{\pi}) \to 0$. Since $\rhomax$ and $\wmax$ are constant, and $\frac{ d^{\pi}_t(s) }{d^{\mu}_t(s)} \to \frac{ d^{\pi}(s) }{d^{\mu}(s)}$, the right hand side of equation above converge to zero, which completes the proof.
\end{proof}

Proof of Corollary \ref{cor:sis}:
\begin{proof}
Since $\frac{d^{\pi}(s,a)}{d^{\mu}(s,a)}$ is bounded by $\rhomax\wmax$ and then $\mathbb{E}_{s,a \sim d^\mu}\left( \frac{d^{\pi}(s,a)}{d^{\mu}(s,a)} \right)^2$ is bounded by $\rhomax^2\wmax^2$. Following from Lemma \ref{lem:converge_stationary_ratio}, there exist $T_0 > 0$ such that for all $t>T_0$, $\mathbb{E}_{s_t,a_t \sim d^\mu_t}\left( \frac{d^{\pi}_t(s_t,a_t)}{d^{\mu}_t(s_t,a_t)} \right)^2 \le 2\mathbb{E}_{s,a \sim d^\mu}\left( \frac{d^{\pi}(s,a)}{d^{\mu}(s,a)} \right)^2 \le 2\rhomax^2\wmax^2$. Then by Theorem \ref{thm:sis_variance}, for $T>T_0$
\begin{align*}
    \var \left(\sum_{t=1}^T \frac{d^{\pi}_t(s_t,a_t)}{d^{\mu}_t(s_t,a_t)} \gamma^{t-1} r_t \right) \le& T \sum_{t=1}^T \gamma^{t-1} \mathbb{E}\left( \frac{d^{\pi}_t(s_t,a_t)}{d^{\mu}_t(s_t,a_t)} \right)^2 \\
    \le& T \sum_{t=1}^{T_0} \gamma^{t-1} \mathbb{E}\left( \frac{d^{\pi}_t(s_t,a_t)}{d^{\mu}_t(s_t,a_t)} \right)^2 + 2T(T-T_0) \rhomax^2\wmax^2 \\
    =&O(T^2)
\end{align*}
\end{proof}

\subsection{Proof of Corollary \ref{cor:asis}}
Now we consider an type of approximate SIS estimators, which plug an approximate density ratio into the SIS estimator. More specifially, we consider it use a function $\wt (s_t,a_t)$ to approximate density ratio  $\frac{d^{\pi}_t(s_t,a_t)}{d^{\mu}_t(s_t,a_t)}$, and construct the estimator as:
\begin{align}
    \siswithwt = \sum_{t=1}^T \wt(s,a) \gamma^{t-1}r_t
\end{align}
This approximate SIS estimator is often biased based on the choice of $\wt(s,a)$, so we consider the upper bound of their mean square error with respect to $T$ and the error of the ratio estimator.

\setcounter{theorem}{6}
\begin{theorem}
\label{thm:appr_sis1}
$\siswithwt$ with $\wt$ such that where $ \mathbb{E}_{\mu} \left(\wt(s_t,a_t) - \frac{d^{\pi}_t(s_t,a_t)}{d^{\mu}_t(s_t,a_t)} \right)^2 \le \mseofwt $
\begin{align}
    \text{MSE} \left( \siswithwt \right) \le 2\var(\sis)+ 2T^2 \mseofwt
\end{align}

\end{theorem}
\begin{proof}
\begin{align}
    \text{MSE} \left( \sum_{t=1}^T \wt(s_t,a_t) \gamma^{t-1}r_t \right) 
    &= \mathbb{E} \left( \sum_{t=1}^T \wt(s_t,a_t) \gamma^{t-1}r_t - v^{\pi}\right)^2  \\
     &= \mathbb{E} \left( \sum_{t=1}^T \wt(s_t,a_t) \gamma^{t-1}r_t - \sis + \sis - v^{\pi}\right)^2  \\
     &\le 2\mathbb{E} \left( \sum_{t=1}^T \wt(s_t,a_t) \gamma^{t-1}r_t - \sis\right)^2 + 2\mathbb{E} \left(\sis - v^{\pi}\right)^2  \\
     &\le 2\mathbb{E} \left( \sum_{t=1}^T \gamma^{t-1}r_t \left(\wt(s_t,a_t) - \frac{d^{\pi}_t(s_t,a_t)}{d^{\mu}_t(s_t,a_t)}\right) \right)^2 +  2\var(\sis) \\
     &\le 2T\sum_{t=1}^T \gamma^{2t-2} \mathbb{E} \left( \wt(s_t,a_t) - \frac{d^{\pi}_t(s_t,a_t)}{d^{\mu}_t(s_t,a_t)}\right)^2 + 2\var(\sis) \\
     &\le 2\var(\sis)  + 2T \sum_{t=1}^T \gamma^{2t-2} \mseofwt \\
     &\le 2\var(\sis)+ 2T^2 \mseofwt
\end{align}
\end{proof}

Proof of Corollary \ref{cor:asis}:
\begin{proof}
By Theorem \ref{thm:appr_sis1} we have that the MSE is bounded by
\begin{align}
    2\var(\sis)+ 2T^2 \mseofwt
\end{align}
According to Corollary \ref{cor:sis}:
\begin{align}
    2\var(\sis)+ 2T^2 \mseofwt = O(T^2) + 2T^2 \mseofwt = O(T^2(1+\mseofwt))
\end{align}
\end{proof}

\section{Return-Conditional IS estimators}
\label{appendix:return_condition}
A natural extension of the conditional importance sampling estimators is to condition on the observed returns $G_t$. 
Precisely we examine the general conditional importance sampling estimator:
\begin{align}
    G_t \mathbb{E}\left[ \rho_{1:t} | \phi_t \right]\enspace ,
\end{align} 
and consider when $\phi_t = G_t$. An analytic expression for $ \mathbb{E}\left[ \rho_{1:t} | G_t \right]$  is not available, but we can model this as a regression problem to predict $\mathbb{E}\left[ \rho_{1:t} | G_t \right]$ given an input $G_t$.
A natural approach is to use 
ordinary least squares (OLS) estimator to estimate $\mathbb{E}\left[ \rho_{1:t}|\phi_t\right]$ viewing $\phi_t$ (or any other statistics $G_t$) as an \textit{input} and $\rho_{1:t}$ as an \textit{output}. While tempting at first glance, we show that this approach produces exactly the same estimates of the expected return as that of the crude importance sampling estimator. 

We start by considering the OLS problem associated with the conditional weights in which we want to find a $\widehat{\theta}$ such that $\phi_t^\top \widehat{\theta} \approx \mathbb{E}\left[ \rho_{1:t}|\phi_t\right]$. 
Let $\Phi \in \mathbb{R}^{n\times 2}$ be the \textit{design} matrix containing the observed returns $G^{(i)}_t$ after $t$ steps and $Y \in \mathbb{R}^{n}$ be the vector of importance ratios $\rho^{(i)}_{1:t}$ for each rollout $i$:
\begin{align*}
Y = \begin{bmatrix}\rho_t^{(0)}\\
\vdots\\
\rho_t^{(N)}
\end{bmatrix}, \enspace \Phi = \begin{bmatrix}G_t^{(0)} & 1\\
\vdots & \vdots \\
G_t^{(N)} & 1
\end{bmatrix} \enspace .
\end{align*}
The OLS estimator for the return-conditional weights is then $\widehat{Y} = \Phi \widehat{\theta}$ and where $\widehat{\theta} \in \mathbb{R}^{2}$ is defined as:
\begin{align*}
\widehat{\theta} = \left(\Phi^\top \Phi\right)^{-1} \Phi^\top Y \enspace .
\end{align*}
We can now use the approximate return-conditional weights to form a Monte Carlo estimate of the expected return under the target policy:
\begin{align}
    \hat{v}_{\text{RCIS}} \equiv \frac{1}{N} \sum_{i=0}^N G_t^{(i)} \widehat{Y}^{(i)}  = \frac{1}{N} [1,0] \Phi^\top \widehat{Y}\enspace , \label{eq:rcis}
\end{align}
where $\widehat{Y}^{(i)} = [G_t^{(i)}, 0]^\top \widehat{\theta}$
and the equality follows from the fact that $\Phi^\top Y = [\sum_{i=1}^n \rho_t^{(i)} G_t^{(i)},  \sum_{i=1}^n \rho_t^{(i)} ]^\top$. Using this observation, we can also express the crude importance sampling estimator with the linear combination $\Phi^\top Y$, where $Y$ now consists of the \textit{true} weights:
\begin{align}
\hat{v}_\text{IS} \equiv \frac{1}{N} [1,0] \Phi^\top Y  \label{eq:vis}\enspace.
\end{align}
Note that equation \eqref{eq:rcis} differs from \eqref{eq:vis} only in the term  $\widehat{Y} = \Phi \widehat{\theta} = \Phi \left(\Phi^\top \Phi\right)^{-1} \Phi^\top Y$ and upon closer inspection, we find that:
\begin{align*}
\Phi^\top \hat{e} = \Phi^\top Y - \Phi^\top \widehat{Y} =  
\Phi^\top \left(Y - \Phi\left(\Phi^\top \Phi\right)^{-1}\Phi^\top Y\right)  = \Phi^\top\left(I - H\right)Y = \mathbf{0} \enspace ,
\end{align*}
where $\hat{e}$ is residual vector  $Y - \widehat{Y}$ and $H = \Phi\left(\Phi^\top \Phi\right)^{-1}\Phi^\top$ is the \textit{hat} matrix. Hence, it follows that the estimate of the expected return made under the crude importance sampling estimator must be identical to the extended estimator which uses approximate return-conditional weights:
\begin{align*}
\hat{v}_\text{is} - \hat{v}_\text{RCIS} = \frac{1}{n}  [1,0] \Phi^\top Y - \frac{1}{n}[1,0] \Phi^\top\widehat{Y} = \frac{1}{n}  [1,0] \left( \Phi^\top Y - \Phi^\top\widehat{Y} \right)  = [1,0] \mathbf{0}= 0\enspace .
\end{align*}
This analysis can be generalized to any conditional importance sampling estimator for which $G_t$ can be expressed as a linear combination of $\phi_t$. For example, rather than conditioning on the final return, we could condition on the return so far (the sum of returns to the present) and use $\phi_t = [r_1, r_2, ... r_t]$ with the coefficient vector $[1, 1, ... 1, 0]$. Similarly, this negative result carries to reward-conditional weights if the immediate reward $r_t$ can be expressed as linear combination of $\phi_t$, including if $\phi_t$ is simply the immediate reward.

\end{document}